\def\year{2022}\relax
\documentclass[letterpaper]{article} % DO NOT CHANGE THIS
\usepackage{aaai22}  % DO NOT CHANGE THIS
\usepackage{times}  % DO NOT CHANGE THIS
\usepackage{helvet}  % DO NOT CHANGE THIS
\usepackage{courier}  % DO NOT CHANGE THIS
\usepackage[hyphens]{url}  % DO NOT CHANGE THIS
\usepackage{graphicx} % DO NOT CHANGE THIS
\usepackage[sort&compress]{natbib}  % DO NOT CHANGE THIS AND DO NOT ADD ANY OPTIONS TO IT
\usepackage{caption} % DO NOT CHANGE THIS AND DO NOT ADD ANY OPTIONS TO IT
\DeclareCaptionStyle{ruled}{labelfont=normalfont,labelsep=colon,strut=off} % DO NOT CHANGE THIS
\usepackage{algorithm}
\usepackage{algorithmic}

\usepackage{newfloat}
\usepackage{listings}
\floatstyle{ruled}
\newfloat{listing}{tb}{lst}{}
\floatname{listing}{Listing}
\nocopyright
\usepackage{amsmath}
\usepackage{amsfonts}
\usepackage{subfigure}
\usepackage{multirow}
\usepackage{makecell}
\usepackage{booktabs}

\usepackage{amsthm}
\newtheorem{theorem}{Theorem}

\title{Automatic Reward Design via Learning Motivation-Consistent Intrinsic Rewards}
\author{Yixiang	Wang\textsuperscript{\rm 1}, Yujing Hu\textsuperscript{\rm 2},
Feng Wu\textsuperscript{\rm 1}, Yingfeng Chen\textsuperscript{\rm 2}
}
\affiliations{
\textsuperscript{\rm 1}University of Sicence and Technology of China,
\textsuperscript{\rm 2}NetEase Fuxi AI Lab \\
yixiangw@mail.ustc.edu.cn, huyujing@corp.netease.com,
wufeng02@ustc.edu.cn, chenyingfeng1@corp.netease.com
}

\begin{document}

\maketitle

\begin{abstract}
  Reward design is a critical part of the application of reinforcement learning, the performance of which strongly depends on how well the reward signal frames the goal of the designer and how well the signal assesses progress in reaching that goal. In many cases, the extrinsic rewards provided by the environment (e.g., win or loss of a game) are very sparse and make it difficult to train agents directly. Researchers usually assist the learning of agents by adding some auxiliary rewards in practice. However, designing auxiliary rewards is often turned to a trial-and-error search for reward settings that produces acceptable results. In this paper, we propose to automatically generate goal-consistent intrinsic rewards for the agent to learn, by maximizing which the expected accumulative extrinsic rewards can be maximized. To this end, we introduce the concept of motivation which captures the underlying goal of maximizing certain rewards and propose the motivation based reward design method. The basic idea is to shape the intrinsic rewards by minimizing the distance between the intrinsic and extrinsic motivations. We conduct extensive experiments and show that our method performs better than the state-of-the-art methods in handling problems of delayed reward, exploration, and credit assignment.
\end{abstract}

\section{Introduction}

In recent years, Reinforcement Learning (RL) has shown great success on many complex single-agent domains \cite{schulman2017proximal}, two-player turn-based games \cite{silver2017mastering}, multi-agent systems \cite{jaderberg2019human}, etc. However, due to the imperfect reward model of increasingly complex environments, many RL applications are deeply involved in reward shaping, which is a process of assisting agents with extra knowledge.

In many applications, the shaping rewards are usually set manually by experts. Here, one challenge is to design appropriate reward signals so that as an agent learns, its behavior approaches and ideally eventually achieves what the application's designer actually desires \cite{sutton2018reinforcement}. Another challenge is to design a reward function that encourages the legitimate behaviors while still being learnable \cite{rlblogpost}. Furthermore, RL agents can discover unexpected ways to make their environments deliver rewards, some of which are cheating through loopholes of reward functions \cite{popov2017data}. Therefore, reward shaping is often turned to a trial-and-error search for rewards setting that produces acceptable results. Potential-Based Reward Shaping (PBRS) \cite{ng1999policy} ensures that the optimal policy of the original problem is not violated by the shaping reward function, which inspired a series of methods for transforming human knowledge into numerical reward functions \cite{knox2009interactively}.

In order to clarify the goal of reward shaping, a useful idea is to divide environmental rewards into extrinsic and intrinsic rewards, which is used in the optimal reward framework \cite{singh2010intrinsically}. Extrinsic rewards define the goal of the problem, while intrinsic rewards provide trainable signals to improve the learning dynamics of the agent and should also be optimized. The original optimal reward framework \cite{singh2010intrinsically} uses an exhaustive search method to find an optimal intrinsic reward function from the reward function space.
\citet{zheng2018learning} proposed a gradient-based method for learning neural network parameterized intrinsic rewards to avoid exhaustive search, but still have several limitations as later shown in our experiments.
\citet{jaderberg2019human} showed that human-designed game points\footnote{e.g., ``I tagged opponent with the flag'' in capture the flag game.} can be used to reduce the search space of shaping reward functions in complex environments. However, the high computational cost of the evolutionary strategy that they used for reward function optimization makes the proposed method inefficient.

In this paper, we propose the \textit{Motivation-Based Reward Design} (MBRD) method which also utilizes prior knowledge such as game points to reduce search space and shares the high computational efficiency of the gradient-based method. The basic idea of our method is to automatically generate goal-consistent intrinsic rewards for the agent, by maximizing which the expected cumulative extrinsic rewards are maximized.
Specifically, we introduce the concept of motivation to capture the underlying goal of maximizing certain rewards and shape intrinsic rewards by minimizing the distance between extrinsic and intrinsic motivations.

Our main contributions are summarized as follow:
\begin{itemize}
  \item Firstly, we propose to combine the game points-based and gradient-based methods to generate intrinsic rewards for the agent, which on one hand can utilize prior knowledge to accelerate learning, and on the other hand avoids high computational cost.
  \item Secondly, instead of relying on the complicated meta-policy gradient method, we propose the novel MBRD method to learn intrinsic rewards based on the principle that intrinsic rewards must provide consistent motivation with the extrinsic rewards.
  \item Lastly, we conduct extensive experiments in the five benchmark domains, where MBRD learns faster and performs better over other baselines, and shows the advantages in handling problems of delayed reward, exploration, and credit assignment.
\end{itemize}

All together, we propose a novel automatic reward design approach via learning motivation-consistent intrinsic rewards, which are useful for many RL applications.

%Our main contributions are as follows: 1.Unlike previous gradient-based works optimize intrinsic rewards by treating which as a part of the policy ~\citep{sorg2010reward,guo2016deep,zheng2018learning}, MBRD proposes a novel way to optimize intrinsic rewards by extracting motivations of extrinsic and intrinsic rewards and minimizing the distance between the motivations; 2.MBRD utilizes prior knowledge such as game points to reduce search space; 3.We test MBRD in foraging, hungry-thirsty, fight monster domains, where MBRD learns faster and performs better over other baselines, which indicates MBRD has advantages in handling problems of delayed reward, exploration, and credit assignment.

%Our main contributions are as follows: 1. MBRD introduces the concept of motivation which captures the underlying goal for maximizing certain rewards; 2. MBRD introduces a novel way to shape intrinsic rewards by minimizing the distance between the intrinsic motivation and the extrinsic motivation; 3. MBRD ensures the expected cumulative extrinsic rewards can be maximized by maximizing expected cumulative intrinsic rewards.

\section{The MBRD Method}

In this section, we propose our MBRD algorithm, which is built on top of the optimal reward framework and can also utilize prior knowledge to reduce the search space. As in the optimal reward framework, the goal of MBRD is to maximize the fitness of agent in the environment by training the agent with parameterized intrinsic rewards.
%We aim to design an interface between environment and agent which makes learning in the new problem is easier than learning in the original problem.

\subsection{Problem Statement}

%Our goals and purposes of training an agent can be well thought of as the maximization of the expected cumulative reward ~\citep{sutton2018reinforcement}.

Assuming that the {\em Markov Decision Process} (MDP) of the underlying problem is $\mathcal{M} = \langle \mathcal{S},\mathcal{A},T,\gamma,R \rangle$, where: $\mathcal{S}$ is the state space, $\mathcal{A}$ is the action space, $T$ is the transition function, $\gamma$ is the discount factor, and $R$ is the reward function. Our goal is to introduce a new MDP $\mathcal{M}' = \langle \mathcal{S},\mathcal{A},T,\gamma,R' \rangle$ by modifying the reward function $R$ and guarantee the optimal policy learned in $\mathcal{M}'$ is still the optimal policy in $\mathcal{M}$. This is also known as policy invariance in PBRS. In other words, the purpose of policy invariance is to make the optimal policy $\pi_{M^{\prime}}^{*}$ in $\mathcal{M}'$ that gets the maximum cumulative reward $E_{a \sim \pi_{M^{\prime}}^{*}}\left[\sum_{t=0}^{\infty} \gamma^{t} R^{in}(s^{t}, a^{t})\right]$, can also get the maximum counterpart $E_{a \sim \pi_{M^{\prime}}^{*}}\left[\sum_{t=0}^{\infty} \gamma^{t} R^{ex}(s^{t}, a^{t})\right]$ in $\mathcal{M}$.

Let $R^{ex}$ denote the reward given by the environment (i.e., the extrinsic reward), and the fitness of agent is the cumulative sum of $R^{ex}$. In the environment we collect $n$-dimensional countable features $\pmb{\rho}:S\times A \rightarrow \mathbb{R}^n$, which represents some special events\footnote{e.g., game points \cite{jaderberg2019human}. }. We do not make any assumption with $\pmb{\rho}$ except they must be countable features. We build our intrinsic reward on the given features $R^{in}_{\phi}=\pmb{w}_{\phi} \cdot \pmb{\rho}$, where $\pmb{w}_{\phi}$ are the weights of the features and $\phi$ are parameters of $\pmb{w}_{\phi}$. Now, we aims to training agent's policy by parameterized intrinsic rewards which are optimized by gradient-based method.

Specifically, the task can be transformed to a two-tier optimization problem as follows.
\begin{itemize}
  \item The inner optimization problem for $\pi_{\theta}$ given $\phi$:
  \begin{equation}
    \pi_{\theta} = \arg\max_{\pi_{\theta}'} J_{\text {inner }}(\pi_{\theta}', \phi), \quad \text{where} \label{eq:inner_update}
  \end{equation}
  \begin{equation}
    J_{\text{inner}}(\pi_{\theta}, \phi) = E_{a \sim \pi_{\theta}}\left[\sum_{t=0}^{\infty} \gamma^{t} R^{in}(s, a | \pmb{\rho}^t, \phi) \right]
  \end{equation}
  \begin{equation}
    R^{in}(s^t, a^t | \pmb{\rho}^t, \phi) = \pmb{w}_{\phi} \cdot \pmb{\rho}(s^t, a^t) \label{eq:intrinsic_reward}
  \end{equation}
  \item The outer optimization problem for $\phi$ given $\pi_{\theta}$:
  \begin{equation}
    \phi = \arg\max_{\phi'} J_{\text{outer}}(\phi', \pi_{\theta}), \quad \text{where} \label{eq:outer_update}
  \end{equation}
  \begin{equation}
    J_{\text{outer}}\left(\phi, \pi_{\theta}\right) = E_{a \sim \pi_{\theta}}^{\phi}\left[\sum_{t=0}^{\infty} \gamma^{t} R^{ex}(s^t, a^t)\right]
  \end{equation}
\end{itemize}

%\begin{equation}
%\begin{aligned}
%J_{\text{inner}}(\pi_{\theta} | \phi) &= E_{a \sim \pi_{\theta}}\left[\sum_{t=0}^{\infty} \gamma^{t}(R^{in}(s, a | \pmb{\rho}^t, \phi))\right] \\
%J_{\text{outer}}\left(\phi | \pi_{\theta}\right) &= E_{a \sim \pi_{\theta}}^{\phi}\left[\sum_{t=0}^{\infty} \gamma^{t} R^{ex}(s^t, a^t)\right] \\
%R^{in}(s^t, a^t | \pmb{\rho}^t, \phi) &= \pmb{w}_{\phi} \cdot \pmb{\rho}(s^t, a^t) \\
%\pi_{\theta}^{\phi} &= \text { optimise }_{\pi_{\theta}}\left(J_{\text {inner }}, \phi\right).
%\end{aligned}
%\end{equation}

%\iffalse
%\begin{equation}
%\begin{aligned}
%J_{\text{inner}}(\pi_{\theta} | \phi)=E_{a \sim \pi_{\theta}}\left[\sum_{t=0}^{\infty} \gamma^{t}(r^{in}(s, a | \pmb{\rho}^t, \phi))\right] ~,~ & %\\
%J_{\text{outer}}\left(\phi | \pi_{\theta}\right)=E_{a \sim \pi_{\theta}}^{\phi}\left[\sum_{t=0}^{\infty} \gamma^{t} r^{ex}(s^t, a^t)\right] \\
%r^{in}(s^t, a^t | \pmb{\rho}^t, \phi)=\pmb{w}_{\phi} \cdot \pmb{\rho}(s^t, a^t) ~,~ & %\\
%\pi_{\theta}^{\phi}=\text { optimise }_{\pi_{\theta}}\left(J_{\text {inner }}, \phi\right).
%\end{aligned}
%\end{equation}
%\fi

Here, the inner optimization maximizes $J_{\text{inner}}$, i.e., the agent's expected discounted accumulative intrinsic rewards. The outer optimization of $J_{\text{outer}}$ can be viewed as a meta-task, in which the meta-reward (extrinsic reward of the environment) is maximized with respect to intrinsic reward schemes, with the inner optimization providing the meta transition dynamics \cite{jaderberg2019human}.

%We don't make any assumptions for $J_{inner}$, so that any reinforcement learning algorithm can be used in $J_{inner}$, and we use Proximal Policy Optimization  ~\citep{schulman2017proximal} algorithm in our experiment.
%Note that the algorithm used in virtual optimization process is also unlimited,
%represent the action or purpose encouraged by extrinsic rewards as a vector
%The result of the virtual update has an intuitive purpose.

\subsection{Algorithm Overview}

The main procedure of our MBRD method is outline in Algorithm \ref{alg:mbrd}. In each training episode, we start with collecting state transitions with the extrinsic reward $r^{ex}$ and features $\rho$ and store them in the replay buffers (Lines 6-11). After that, we update the policy $\pi_{\theta}$ and weights of the intrinsic reward $w_{\phi}$ in batches given the transitions from the replay buffers (Lines 12-17). We iteratively repeat the processes above until it is timeout and return the final policy.

\begin{algorithm}[t]
\caption{Motivation-Based Reward Design (MBRD)}
\label{alg}
\begin{algorithmic}[1]
\STATE Initialize intrinsic reward parameters $w_{\phi}$
\STATE Initialize buffer $\mathcal{D}$ for policy update
\STATE Initialize buffer $\mathcal{D'}$ for $w_{\phi}$ update
\REPEAT
\STATE Initialize environment at the start of an episode
\REPEAT
\STATE $a \sim \pi_{\theta}\left(a | s\right)$
\STATE Execute $a$ in the environment
\STATE Store $(s,a,s',\rho)$ in replay buffer $\mathcal{D}$
\STATE Store $(s,a,r^{ex},s',\rho)$ in replay buffer $\mathcal{D'}$
\UNTIL{episode end.}
\IF{time to update}
\STATE Collect batch $\mathcal{T}$ from $\mathcal{D}$, and compute intrinsic reward $R^{in}$ using $w_{\phi}$ by Equation \ref{eq:intrinsic_reward}
\STATE Update policy $\pi_{\theta}$ using $\mathcal{T}$ by Equation \ref{eq:inner_update}
\STATE Collect batch $\mathcal{T'}$ from $\mathcal{D'}$, and compute Monte-Carlo returns using extrinsic and intrinsic reward respectively for all $(s,a) \sim \mathcal{T'}$.
\STATE Update intrinsic reward parameters $w_{\phi}$ with $\mathcal{T'}$ by Equation \ref{intinsic_reward_update}
\ENDIF
\UNTIL{timeout.}
\RETURN policy $\pi_{\theta}$
\end{algorithmic} \label{alg:mbrd}
\end{algorithm}

In more details,  the inner optimization problem is solved by computing the intrinsic rewards $R^{in}$ by Equation \ref{eq:intrinsic_reward} and then updating the policy $\pi_{\theta}$ with the objective function in Equation \ref{eq:inner_update} by gradient-based method (Lines 13-14). This is straightforward because the parameterized intrinsic rewards are denser and more informative given the features. Indeed, this can be done using some standard RL techniques.

The main challenge here is how to solve the outer optimization problem, in which the best weights for the intrinsic rewards $w_{\phi}$ is updated (Lines 15-16). In Equation \ref{eq:outer_update}, we want to compute the parameters $\phi$ of the weights by maximizing the expected values of the extrinsic rewards $R^{ex}$ (i.e., the real rewards of the environment). However, it is challenging to do this because we may need to search the huge parameter space of $\phi$ by: firstly computing a policy $\pi_{\theta}$ for each $\phi$ in the inner optimization, then estimating $J_\text{outer}$ with $\pi_{\theta}$, and finally selecting the best $\phi$ with the maximum $J_\text{outer}$.

Against this background, we propose our novel approach to the outer optimization problem. The basic idea is to make a connect between the extrinsic reward $R^{ex}$ and the intrinsic reward $R^{in}$ under the concept of motivation. In the following section, we will describe what motivation means and how to map rewards to motivations.

\subsection{Mapping Rewards to Motivation}

\begin{figure}[t]
\centering
\includegraphics[width=.6\linewidth]{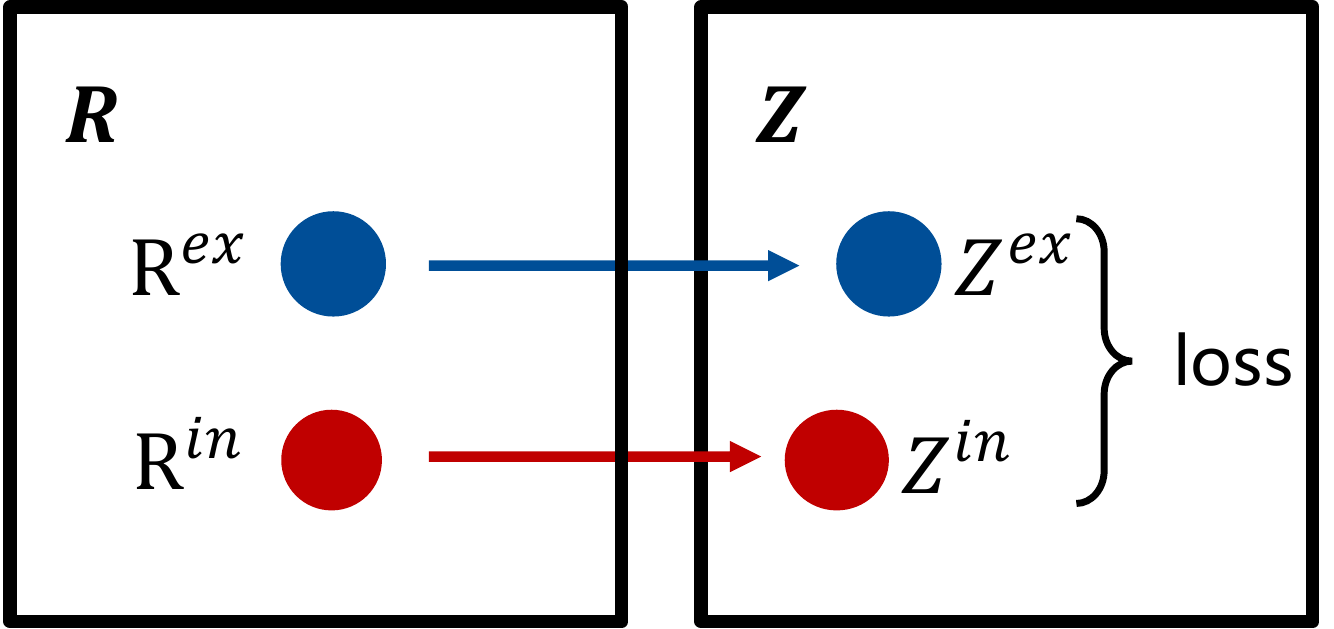}
\caption{Illustration of the basic idea of the MBRD method.}
\label{motivation}
\end{figure}

We introduce a latent variable $z$ to model what behaviors are encouraged by a series of reward signals. This is called {\em motivation} because it motivates the agent towards certain behaviors. Notice that motivation is different from sub-goal because it does not necessarily indicate an achievable results.

We denote the space of the latent variable $z$ by $\mathcal{Z}$ and let $z_{\text{ex}}$ and $z_{\text{in}}$ denote the motivations of improving the expected accumulative extrinsic and intrinsic rewards, respectively. As shown in Figure \ref{motivation}, the basic idea of our method is to minimize the motivation distance between the intrinsic and extrinsic reward functions (i.e., the distance between $z_{\text{in}}$ and $z_{\text{ex}}$). Now the key problem is to find appropriate features which can quantify such two motivations.

For a given reward function, the gradient of the expected accumulative rewards with respect to the policy parameter $\theta$ (i.e., policy gradient) matches well the meaning of motivation. Suppose the policy parameter $\theta$ is directly optimized according to the intrinsic and extrinsic rewards, we can define two virtual learning objectives as:
\begin{equation}
\begin{aligned}
\operatorname{J_{virtual}^{ex}}(\theta) &= E_{a \sim \pi_{\theta}}^{\phi}\left[\sum_{t=0}^{\infty} \gamma^{t} R^{ex}(s^t, a^t)\right] \\
\operatorname{J_{virtual}^{in}}(\theta) &= E_{a \sim \pi_{\theta}}\left[\sum_{t=0}^{\infty} \gamma^{t} R^{in}(s^t, a^t | \pmb{\rho}^t, \phi)\right]. \label{vex-vin}
\end{aligned}
\end{equation}
Here, we assume that there are two virtual agents, where one is motivated by the extrinsic rewards $R^{ex}$ and the other is motivated by the intrinsic rewards $R^{in}$. Intuitively, given some random $\phi$, the virtual agent with $R^{ex}$ and the virtual agent with $R^{in}$ will behavior differently. In other words, we say that they have different motivations.

According to the policy gradient theorem, we have:
\begin{equation}
\begin{aligned}
\nabla_{\theta} \operatorname{J_{virtual}^{ex}}(\theta) &= E_{\pi_{\theta}}\left[ \nabla_{\theta} \log \pi_{\theta}(s, a) G^{ex}\right] \\[5pt]
\nabla_{\theta} \operatorname{J_{virtual}^{in}}(\theta) &= E_{\pi_{\theta}}\left[ \nabla_{\theta} \log \pi_{\theta}(s, a) G^{in}_{\phi}\right] %\\
\end{aligned}
\end{equation}
where $G^{ex}$ and $G^{in}$ are Monte-Carlo returns of the intrinsic and extrinsic rewards respectively.

Now, we use the gradient in virtual update process to represent extrinsic and intrinsic motivations:
\begin{equation}
\pmb{z}_{ex} := \nabla_{\theta} \operatorname{J_{virtual}^{ex}}(\theta) ~,\quad
\pmb{z}_{in} := \nabla_{\theta} \operatorname{J_{virtual}^{in}}(\theta).
\end{equation}
Notice that the motivations of the two virtual agents are identical if they have the same policy. On the other hand, they will result in the same policy if their motivations are identical. Therefore, the policy invariance can be achieved by minimizing the distance of the two motivations. Next, we describe how to measure the distance of two motivations.

\subsection{Distance Measure of Motivations}

The next step is to measure the distance between these motivations. This is a new problem since the scales of the extrinsic rewards and the intrinsic rewards may be different. Hence it is inappropriate to require the geometric distance of the two motivations to be as close as possible. Instead, we use the angle difference of motivations to measure the distance between motivations. Intuitively, this measures the tendency of the motivations.

Assume $\zeta$ is the angle of $\pmb{z}_{ex}$ and $\pmb{z}_{in}$. In order to shape intrinsic rewards whose $\pmb{z}_{in}$ should be similar with $\pmb{z}_{ex}$, a straightforward method is to minimize the angle $\zeta$, which is equivalent to maximizing $\cos{\zeta}$. However, directly maximizing $\cos \zeta$ is difficult. Therefore, we choose to maximize an approximation of $\cos{\zeta}$ as below:
\begin{equation}
\begin{aligned}
\mathop{\arg\min}_{\phi} \zeta &= \mathop{\arg\max}_{\phi} \cos{\zeta} \\
&= \mathop{\arg\max}_{\phi} \frac{\pmb{z}_{ex}\cdot\pmb{z}_{in}}{||\pmb{z}_{ex}||\ ||\pmb{z}_{in}||} \\
&= \mathop{\arg\max}_{\phi} \frac{\pmb{z}_{ex}\cdot\pmb{z}_{in}}{||\pmb{z}_{in}||}.\label{e6}
\end{aligned}
\end{equation}

In practice, we use $(\pmb{z}_{ex}\cdot\pmb{z}_{in} - \beta||\pmb{z}_{in}||)$ to approximate Equation \ref{e6} because it is easier to compute and we have
\begin{equation}
  \pmb{z}_{ex}\cdot\pmb{z}_{in} - \beta||\pmb{z}_{in}|| \propto \frac{\pmb{z}_{ex}\cdot\pmb{z}_{in}}{||\pmb{z}_{in}||}
  \propto \cos{\zeta} \propto \zeta
\end{equation}
where $\beta$ is the tuning parameter for the regularization term.

%&\approx \mathop{\arg\max}_{\phi} \pmb{z}_{ex}\cdot\pmb{z}_{in} - \beta||\pmb{z}_{in}||\\
%The last equation is a trick in our practice which modifies the constraint $||\pmb{z}_{in}||$ to new constraint $||\pmb{w}_{\phi} - \pmb{w}_{init}||^2$. The benefit is not only reduce the compute complex, but also maintain the weak activate of temporarily useless intrinsic rewards.

\subsection{Optimizing Outer Objective}

Instead of optimizing the final objective of the meta-task $J_{\text{outer}}(\phi, \pi_{\theta})$ by computing the meta-policy gradient $\nabla_{\phi} J_{\text{outer}}(\phi, \pi_{\theta})$, which is challenging, we propose a much simpler method to optimize $\phi$ based on the above equations for measuring the motivation distance. The basic idea is to optimize $\phi$ in the direction of minimizing the distance between the intrinsic and extrinsic motivations.

Note that the policy $\pi_{\theta}$ is optimized according to the intrinsic rewards. When the intrinsic and extrinsic motivations are close to each other, optimizing the policy will lead to the optimization of the expected accumulative extrinsic rewards. Therefore, we define a new objective for optimizing $\phi$, which is simpler than optimizing $J_{\text{outer}}$ directly as follow:
\begin{equation}
\begin{aligned}
J_{\text{o}}(\phi) & = \pmb{z}_{ex}\cdot\pmb{z}_{in} - \beta||\pmb{z}_{in}|| \\
& = \nabla_{\theta} \operatorname{J_{virtual}^{ex}} \cdot
\nabla_{\theta} \operatorname{J_{virtual}^{in}} - \beta||\pmb{z}_{in}||
\label{jo}
\end{aligned}
\end{equation}
where $\nabla_{\theta} \operatorname{J_{virtual}^{ex}}$ and $\nabla_{\theta} \operatorname{J_{virtual}^{in}}$ are defined in Equation \ref{vex-vin}.

The purpose of maximizing $J_{\text{o}}$ is to minimize the distance between extrinsic and intrinsic motivations. Hence, we update $J_{\text{o}}(\phi)$ by one step of gradient ascent as below:
\begin{equation}\label{intinsic_reward_update}
 \nabla_{\phi} J_{\text{o}} = \nabla_{\theta} \operatorname{J_{virtual}^{ex}} \cdot
\nabla_{\phi}\nabla_{\theta} \operatorname{J_{virtual}^{in}} - \beta\nabla_{\phi}||\pmb{z}_{in}||,
\end{equation}
where we have:
\begin{equation}
\begin{aligned}
\nabla_{\phi}\nabla_{\theta} \operatorname{J_{virtual}^{in}}
 = & \sum_{(s, a) \sim \mathcal{T'}} \nabla_{\theta} \log \pi_{\theta}(s, a) \nabla_{\phi} G^{in}_{\phi} \\
 = & \sum_{(s, a) \sim \mathcal{T'}} \nabla_{\theta} \log \pi_{\theta}(s, a) \\
 & \qquad \cdot \sum_{i=t}^{\infty} \gamma^{i-t} \nabla_{\phi} r_{\phi}^{in}\left(s_{i}, a_{i}\right) \\
 = & \sum_{(s, a) \sim \mathcal{T'}} \nabla_{\theta} \log \pi_{\theta}(s, a) \\
 & \qquad \cdot \sum_{i=t}^{\infty} \gamma^{i-t} \nabla_{\phi} \pmb{w}_{\phi} \cdot \pmb{\rho}\left(s_{i}, a_{i}\right) %\\
\end{aligned}
\end{equation}

%The main procedures of MBRD method are outlined in Algorithm \ref{alg}.
As shown in Algorithm \ref{alg:mbrd}, we use Equation \ref{intinsic_reward_update} to update the intrinsic reward parameters $\phi$. Note that the policy update (Line 14) and intrinsic reward parameter update (Line 16) processes are independent of each other in our method.

\begin{theorem}
  The outer optimization problem of Equation \ref{eq:outer_update} is solved by maximizing $J_{\text{o}}(\phi)$ defined in Equation \ref{jo}.
\end{theorem}
\begin{proof}
In Equation \ref{eq:outer_update}, the goal is to find the parameters $\phi$ that maximizing $J_\text{outer}$. Note that $\operatorname{J_{virtual}^{ex}}$ is equivalent to $J_\text{outer}$. In what follows, we will show that the gradient $\nabla_{\phi} J_{\text{o}}$ is directly proportional to the gradient $\nabla_{\phi} \operatorname{J_{virtual}^{ex}}$ and thereby the gradient $\nabla_{\phi} J_\text{outer}$ as below:
\begin{equation}
  \nabla_{\phi} J_{\text{o}} \propto \nabla_{\phi} \operatorname{J_{virtual}^{ex}} = \nabla_{\phi} J_\text{outer}
\end{equation}

If we apply once policy update with the intrinsic rewards, which updates policy parameters $\theta$ to $\theta'$:
\begin{equation}
\begin{aligned}
\theta' &= \theta + \alpha \nabla_{\theta} \operatorname{J_{virtual}^{in}} \\
&= \theta + \alpha \sum_{(s, a) \sim \mathcal{T'}} \nabla_{\theta} \log \pi_{\theta}(s, a) G^{in}_{\phi}.
\end{aligned}
\end{equation}
Then we use extrinsic reward to maximize the expected cumulative extrinsic reward:
\begin{equation}
\begin{aligned}
\nabla_{\phi} \operatorname{J_{virtual}^{ex}}
&= \nabla_{\theta'} \operatorname{J_{virtual}^{ex}} \nabla_{\phi} \theta' \\
&= \nabla_{\theta'} \operatorname{J_{virtual}^{ex}} \nabla_{\phi}(\theta + \alpha \nabla_{\theta} \operatorname{J_{virtual}^{in}}) \\
&= \alpha \nabla_{\theta'} \operatorname{J_{virtual}^{ex}} \nabla_{\phi} \nabla_{\theta} \operatorname{J_{virtual}^{in}}.%\\
\label{proof}
\end{aligned}
\end{equation}
Note that the result of Equation \ref{proof} is proportional to the first term of Equation \ref{intinsic_reward_update} and the second term is a regularization of $\phi$. This concludes that the parameters $\phi$ of maximizing $J_{\text{o}}$ is also maximizing $J_\text{outer}$ and thus the theorem is proved.

%And the second part, $\beta\nabla_{\phi}||\pmb{w}_{\phi} - \pmb{w}_{init}||^2$, has shown significant effect in our experiment. It not only suppresses the value explosion caused by the excessive growth of $\pmb{w}_{\phi}$, but also makes the useless features in the early stage remain weakly active in the late stage so that they can be activated. We will explain the detail of that effect in our experiment part.

\end{proof}

\begin{figure}[t]
    \centering
    \subfigure[Foraging]{
    \begin{minipage}[t]{0.33\linewidth}
    \centering
    \includegraphics[width=\linewidth]{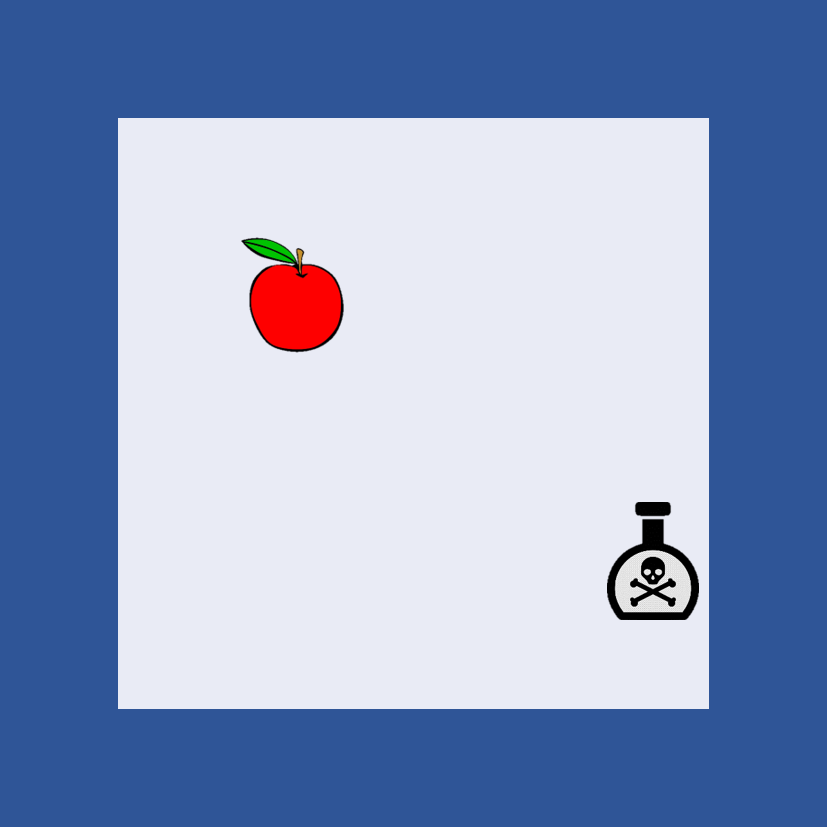}
    %\caption{fig1}
    \end{minipage}%
    }%
    \subfigure[Hungry-Thirsty]{
    \begin{minipage}[t]{0.33\linewidth}
    \centering
    \includegraphics[width=\linewidth]{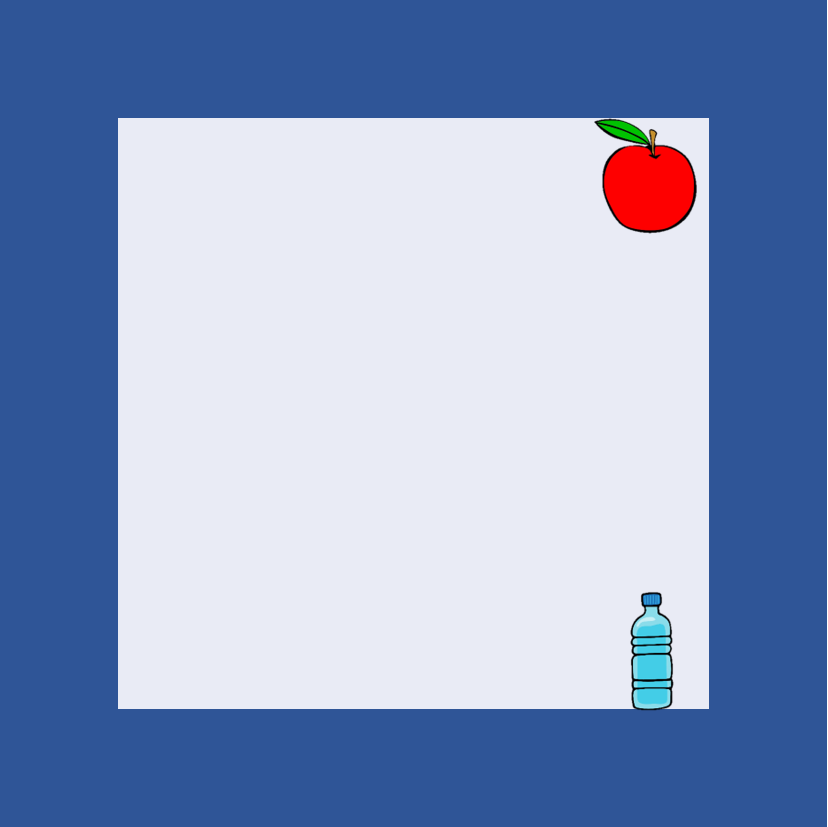}
    %\caption{fig2}
    \end{minipage}%
    }%
    \subfigure[Fight monster]{
    \begin{minipage}[t]{0.33\linewidth}
    \centering
    \includegraphics[width=\linewidth]{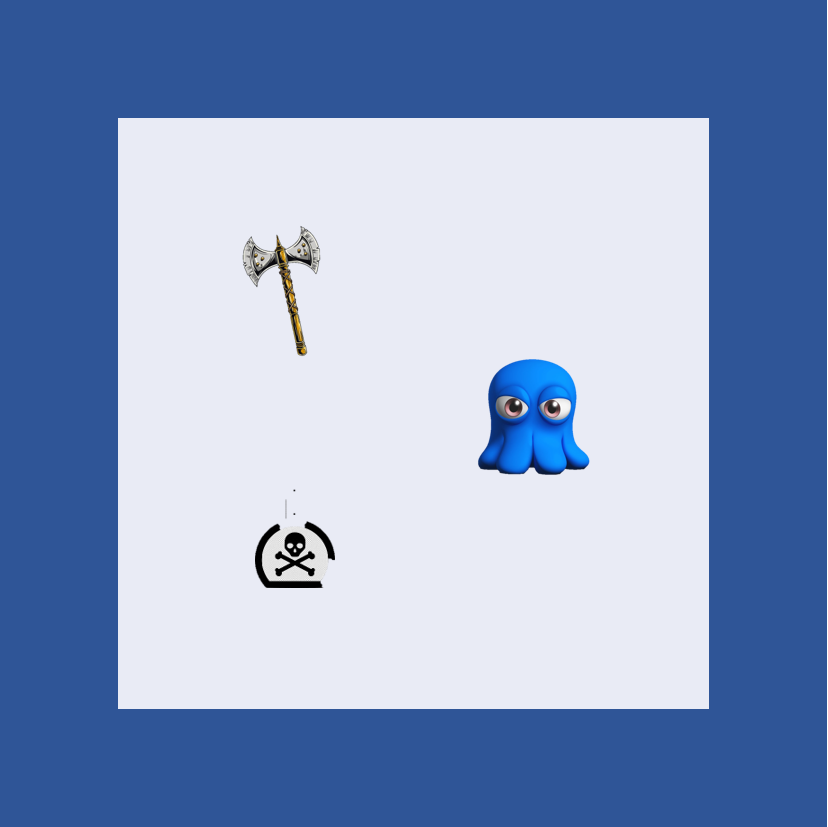}
    %\caption{fig2}
    \end{minipage}
    }
    \centering
    \caption{Illustration of the grid-world domains.}
    \label{domain}
\end{figure}

\begin{figure*}[t]
    \centering
    \subfigure[Foraging]{
    \begin{minipage}[t]{0.32\linewidth}
    \centering
    \includegraphics[width=\columnwidth]{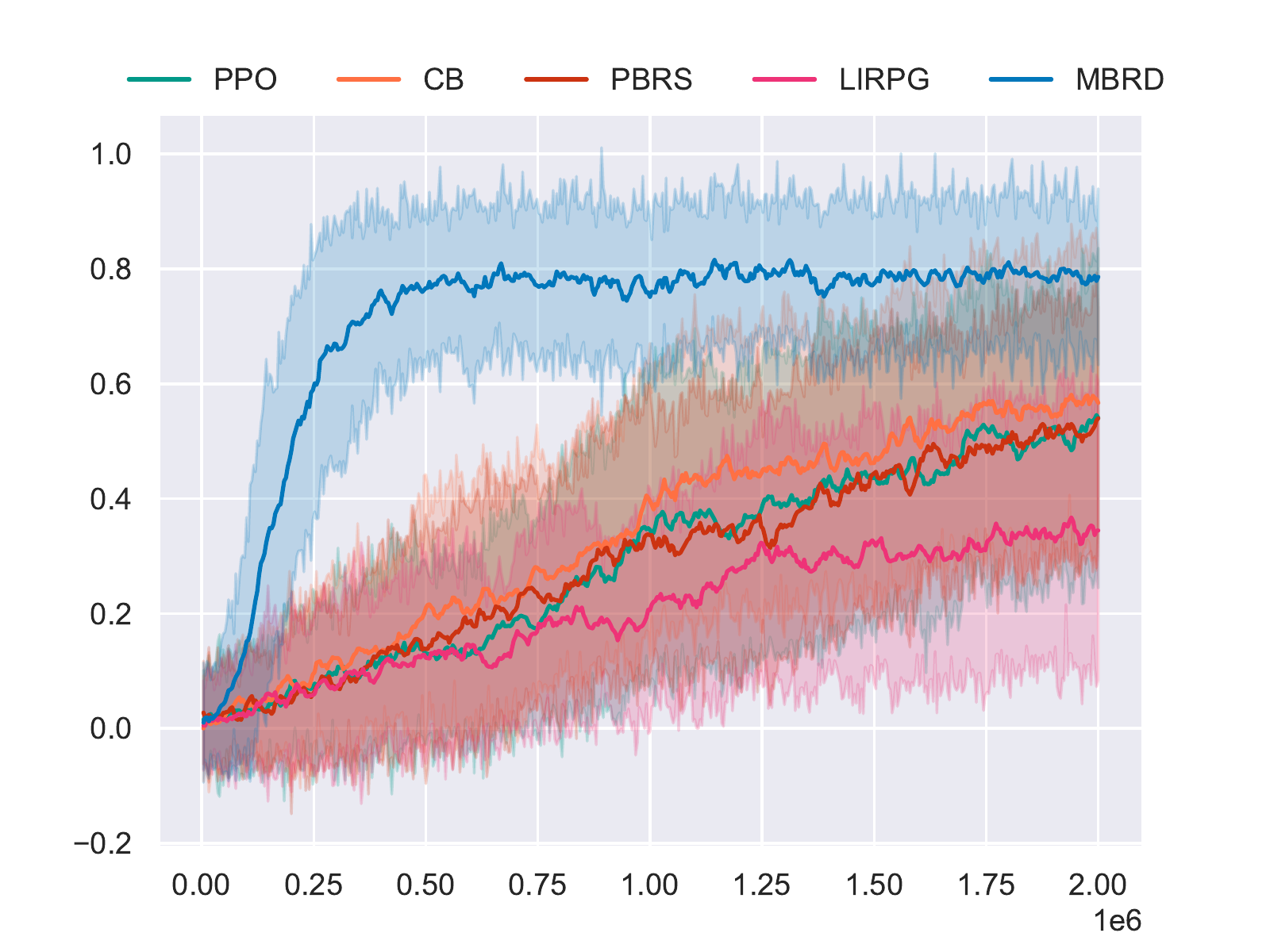}
    \label{foraging}
    \end{minipage}
    }
    \subfigure[Hungry-Thirsty]{
    \begin{minipage}[t]{0.32\linewidth}
    \centering
    \includegraphics[width=\columnwidth]{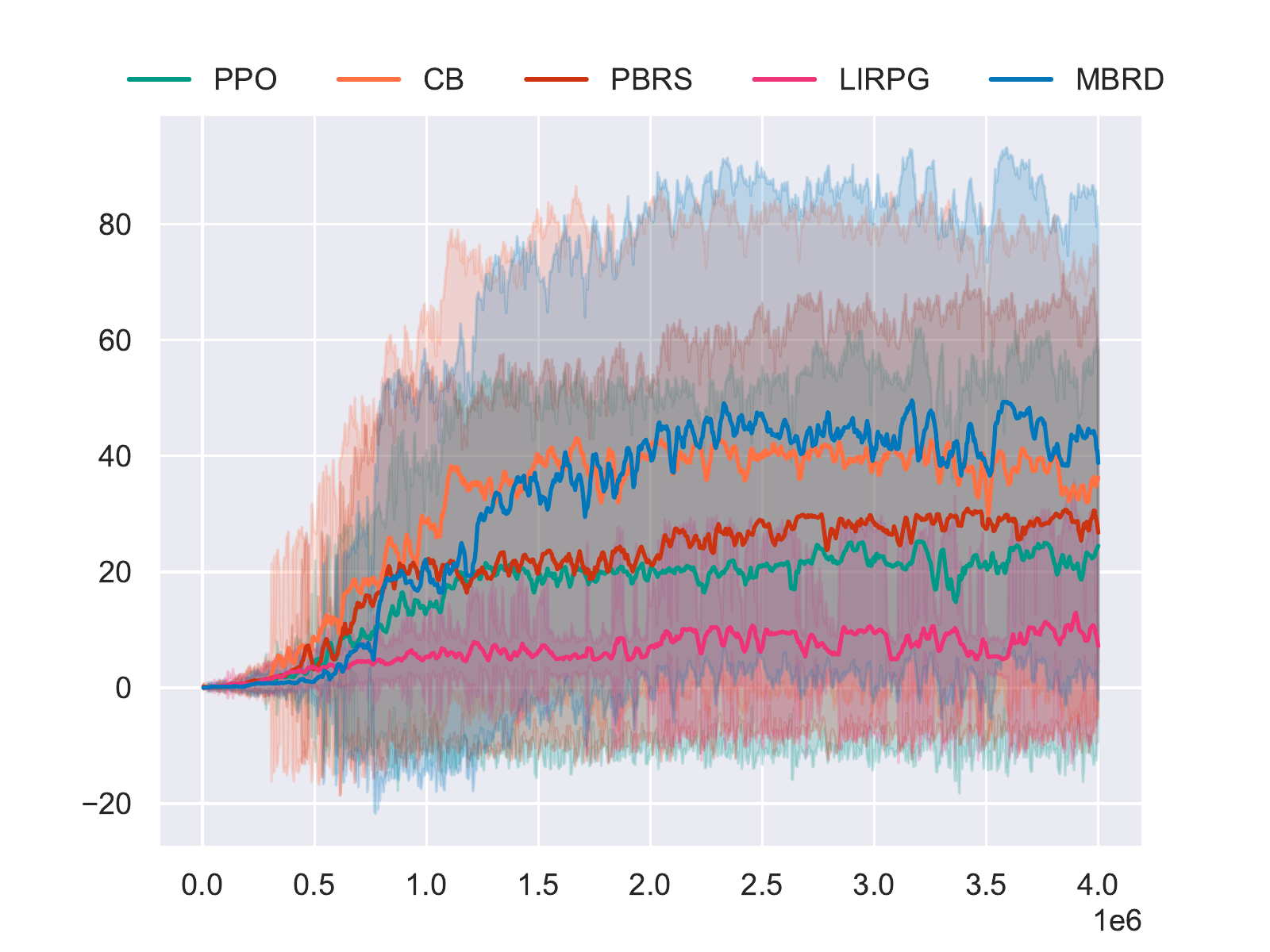}
    \label{hungry}
    \end{minipage}
    }
    \subfigure[Fight Monster]{
    \begin{minipage}[t]{0.32\linewidth}
    \centering
    \includegraphics[width=\columnwidth]{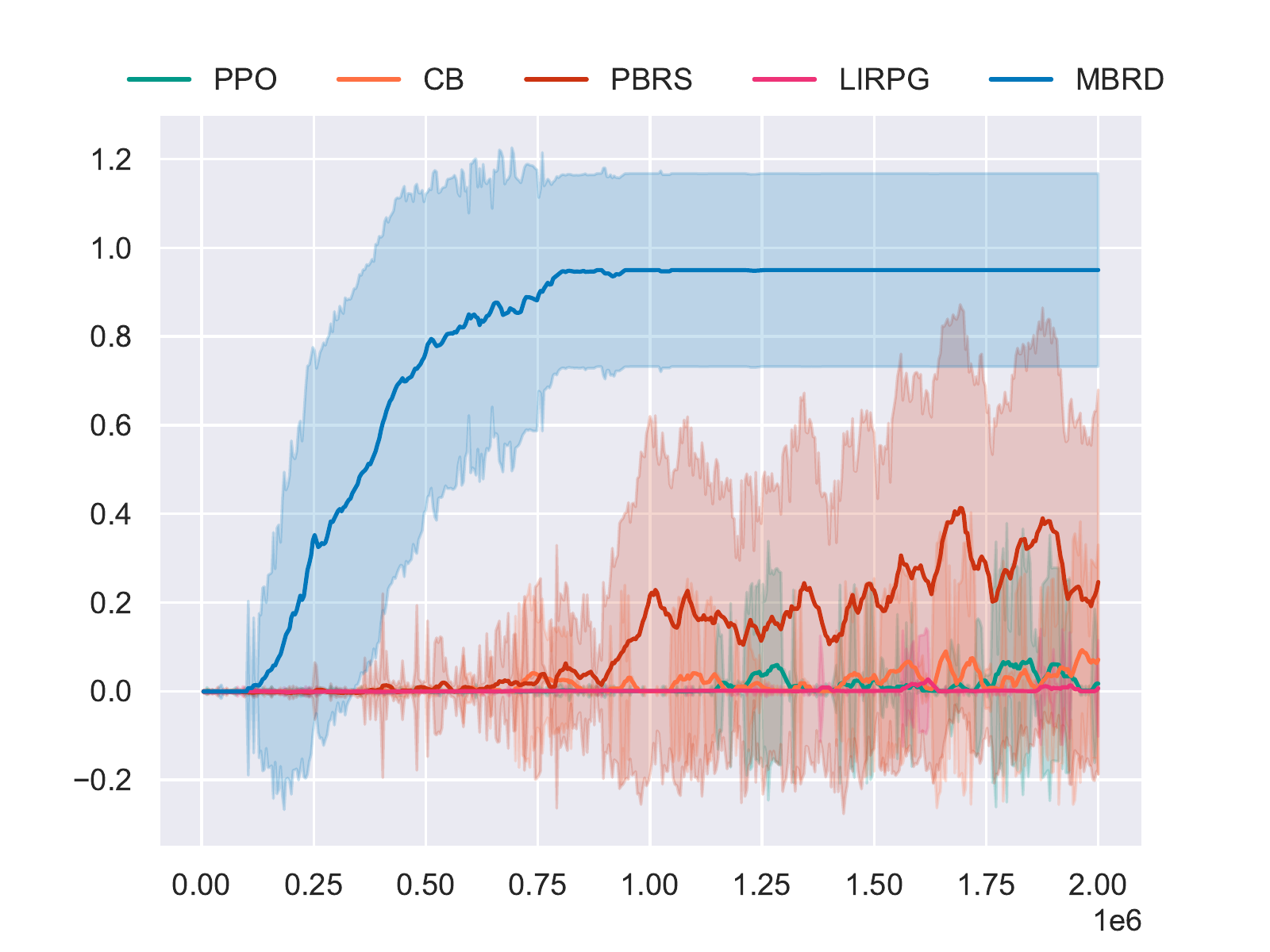}
    \label{reason}
    \end{minipage}
    }
    \subfigure[Hopper-v2]{
    \begin{minipage}[t]{0.32\linewidth}
    \centering
    \includegraphics[width=\columnwidth]{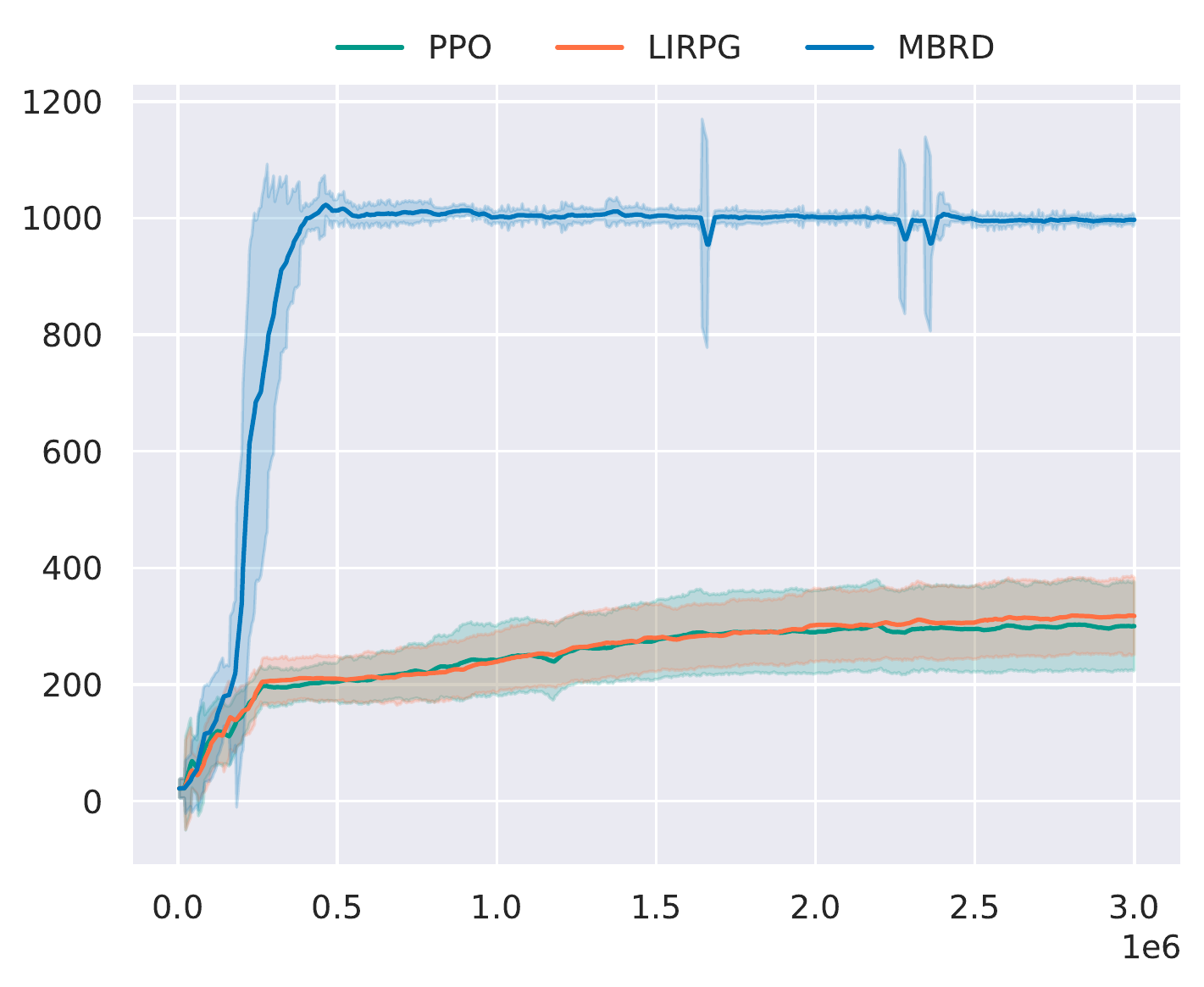}
    \label{hopper}
    \end{minipage}
    }
    \subfigure[Swimmer-v2]{
    \begin{minipage}[t]{0.32\linewidth}
    \centering
    \includegraphics[width=\columnwidth]{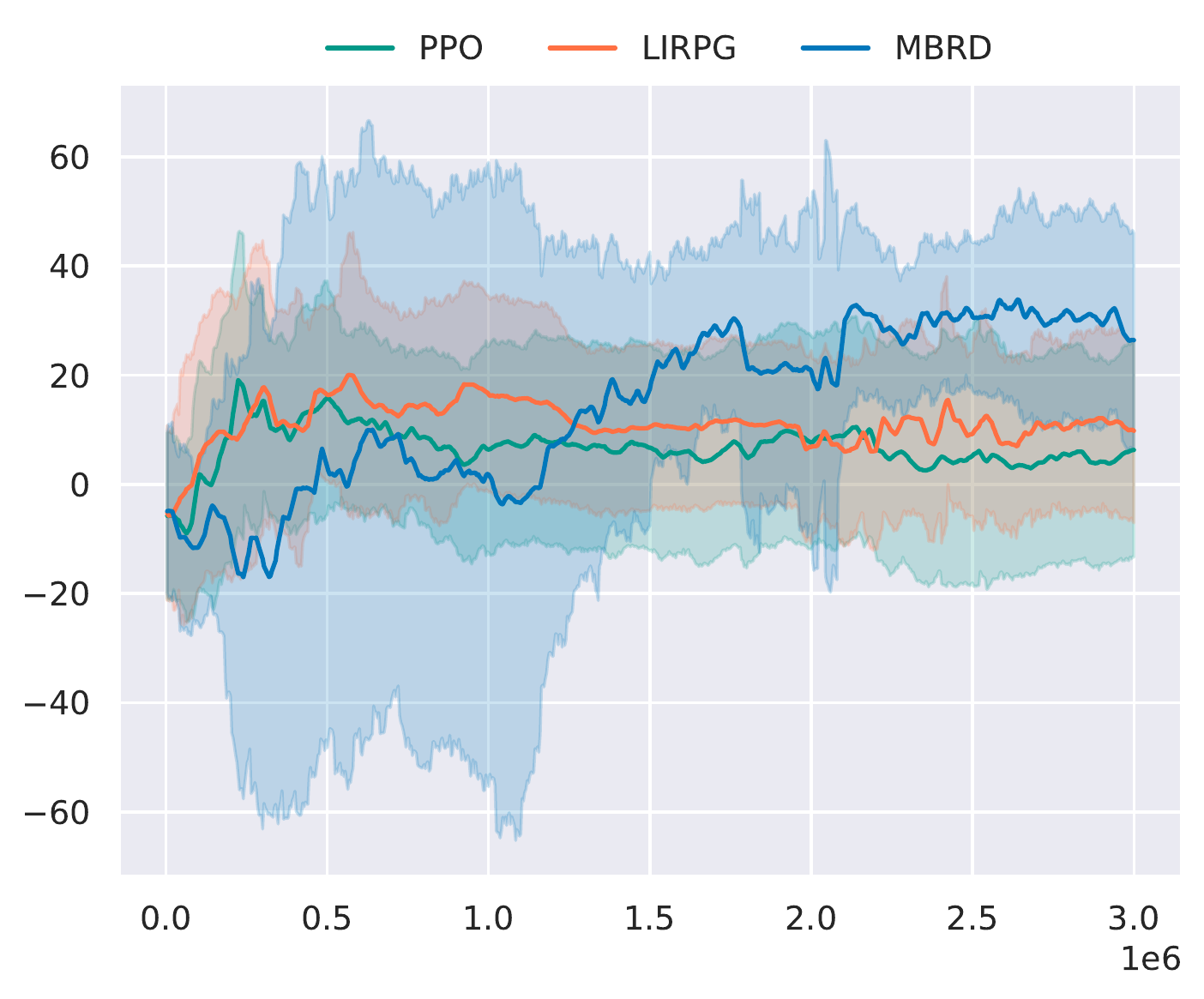}
    \label{swimmer}
    \end{minipage}
    }
    \caption{Experimental results of the grid-world and Mujoco domains (The x-axis is the number of training steps, the y-axis is the average cumulative extrinsic reward, and the shadow area is the standard deviation).}
    \label{performance}
\end{figure*}

\begin{figure*}[t]
    \centering
    \subfigure[Foraging]{
    \begin{minipage}[t]{0.32\linewidth}
    \centering
    \includegraphics[width=\linewidth]{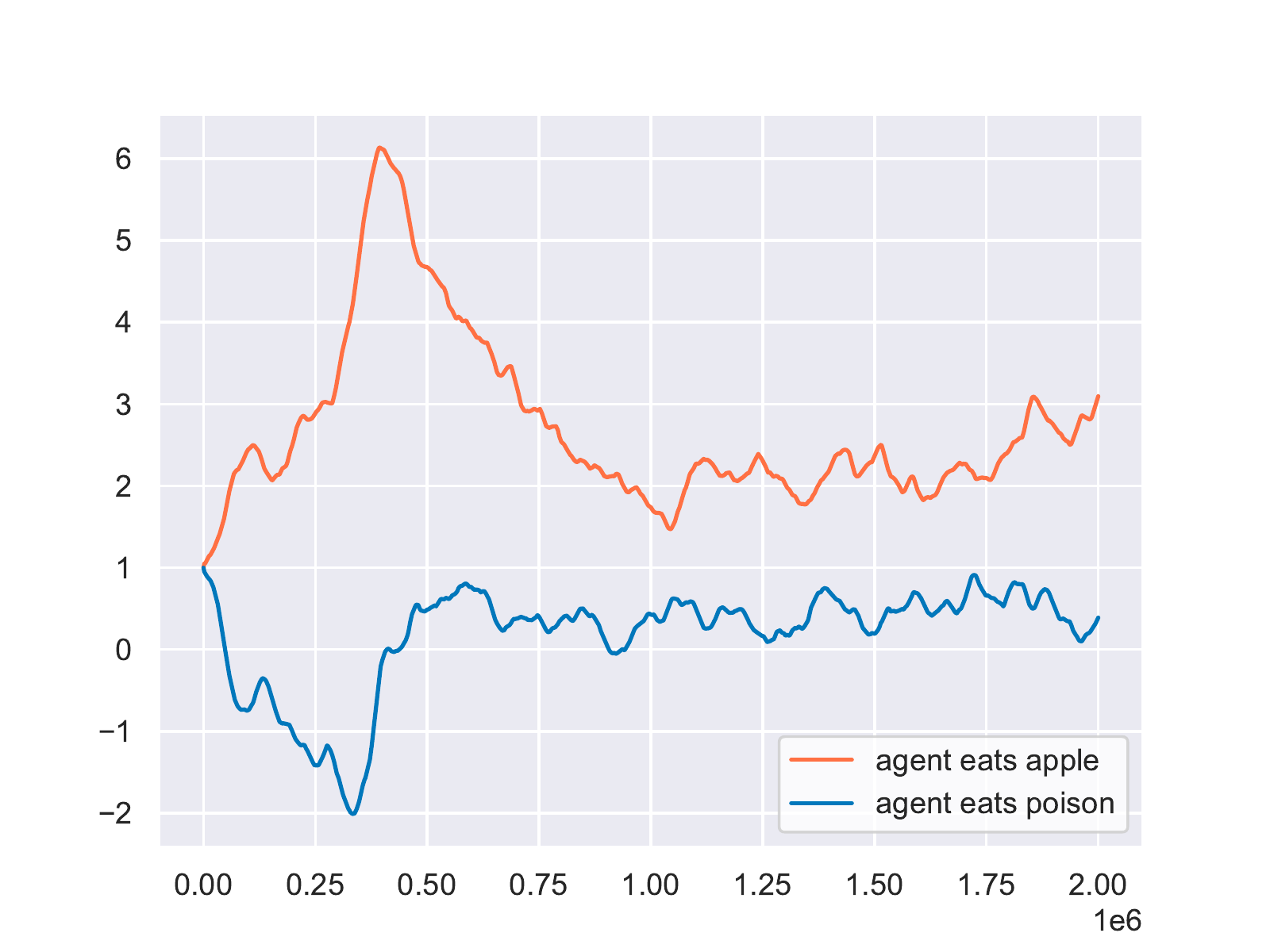}
    \label{foraging_w}
    \end{minipage}
    }
    \subfigure[Hungry-Thirsty]{
    \begin{minipage}[t]{0.32\linewidth}
    \centering
    \includegraphics[width=\linewidth]{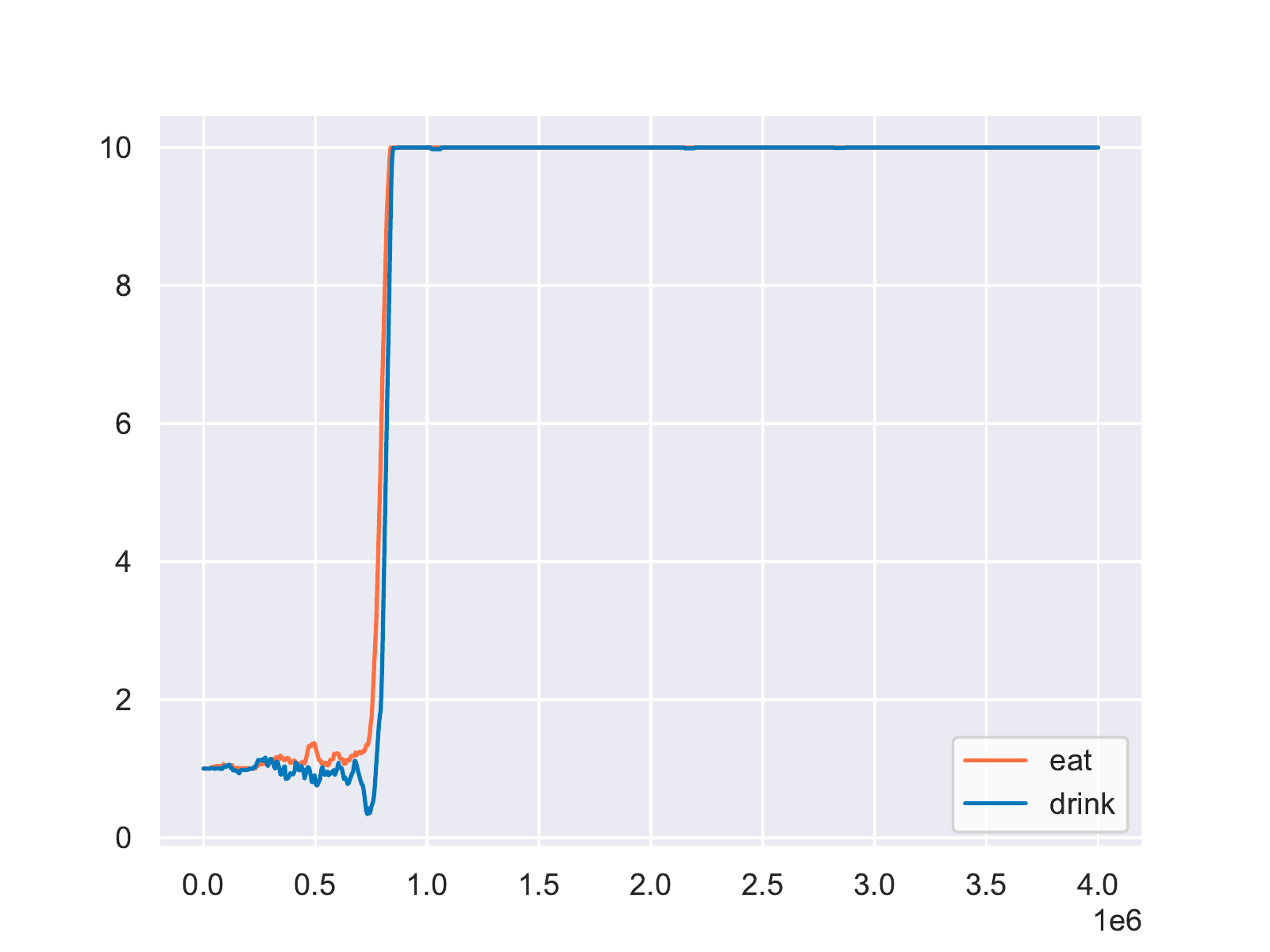}
    \label{hungry_w}
    \end{minipage}
    }
    \subfigure[Fight Monster]{
    \begin{minipage}[t]{0.32\linewidth}
    \centering
    \includegraphics[width=\linewidth]{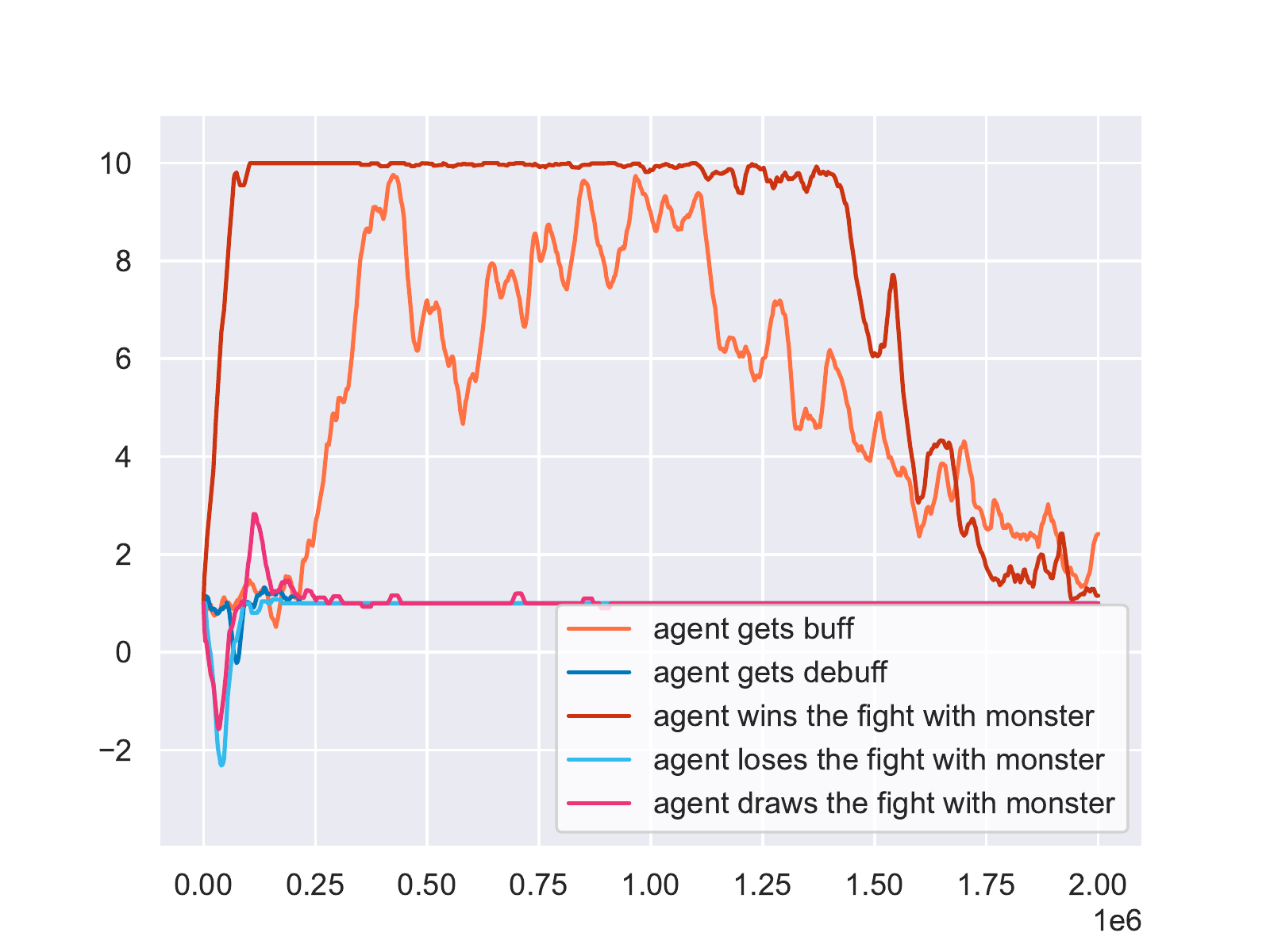}
    \label{reason_w}
    \end{minipage}
    }
    \subfigure[Hopper-v2]{
    \begin{minipage}[t]{0.45\linewidth}
    \centering
    \includegraphics[width=\linewidth]{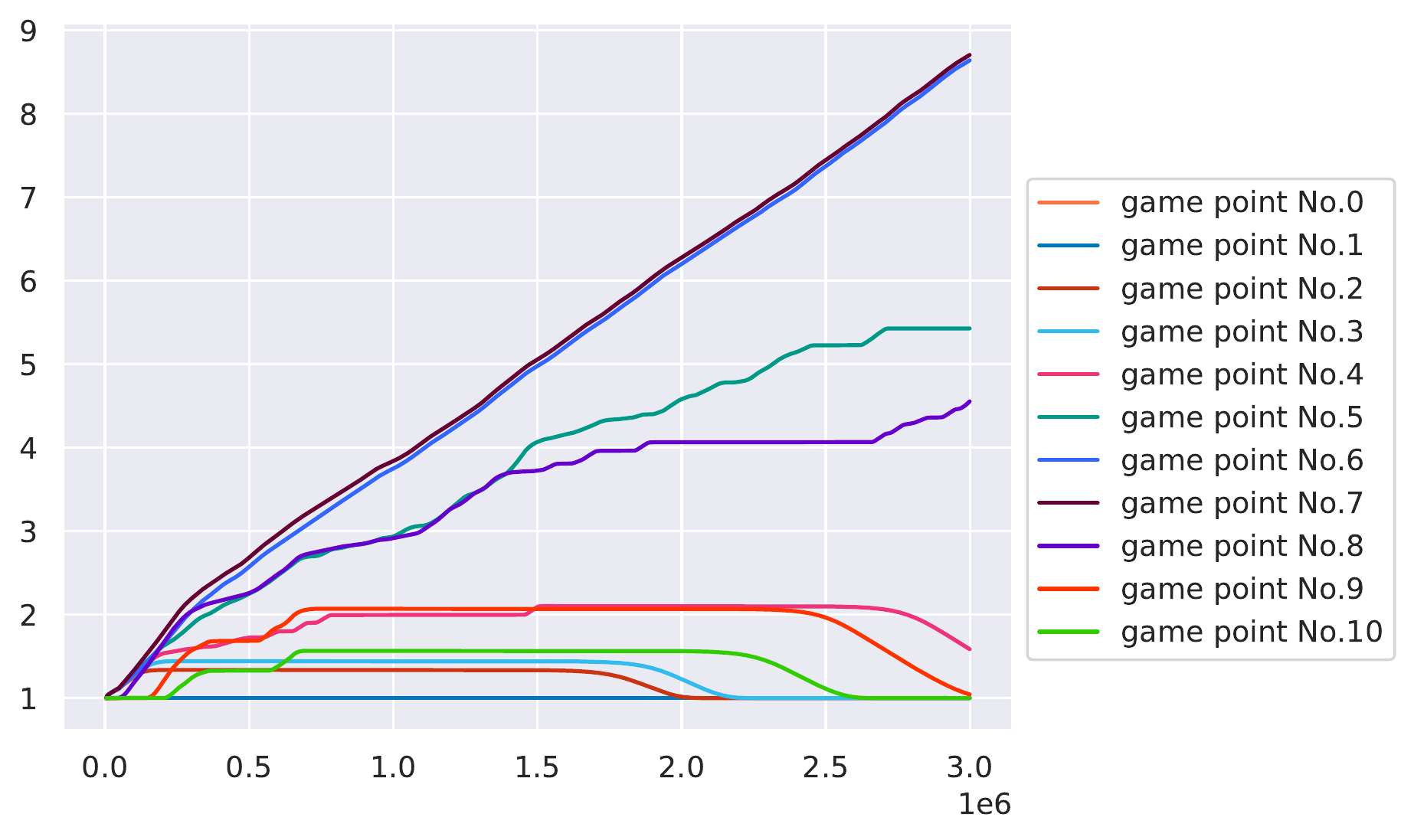}
    \label{hopper_w}
    \end{minipage}
    }
    \subfigure[Swimmer-v2]{
    \begin{minipage}[t]{0.45\linewidth}
    \centering
    \includegraphics[width=\linewidth]{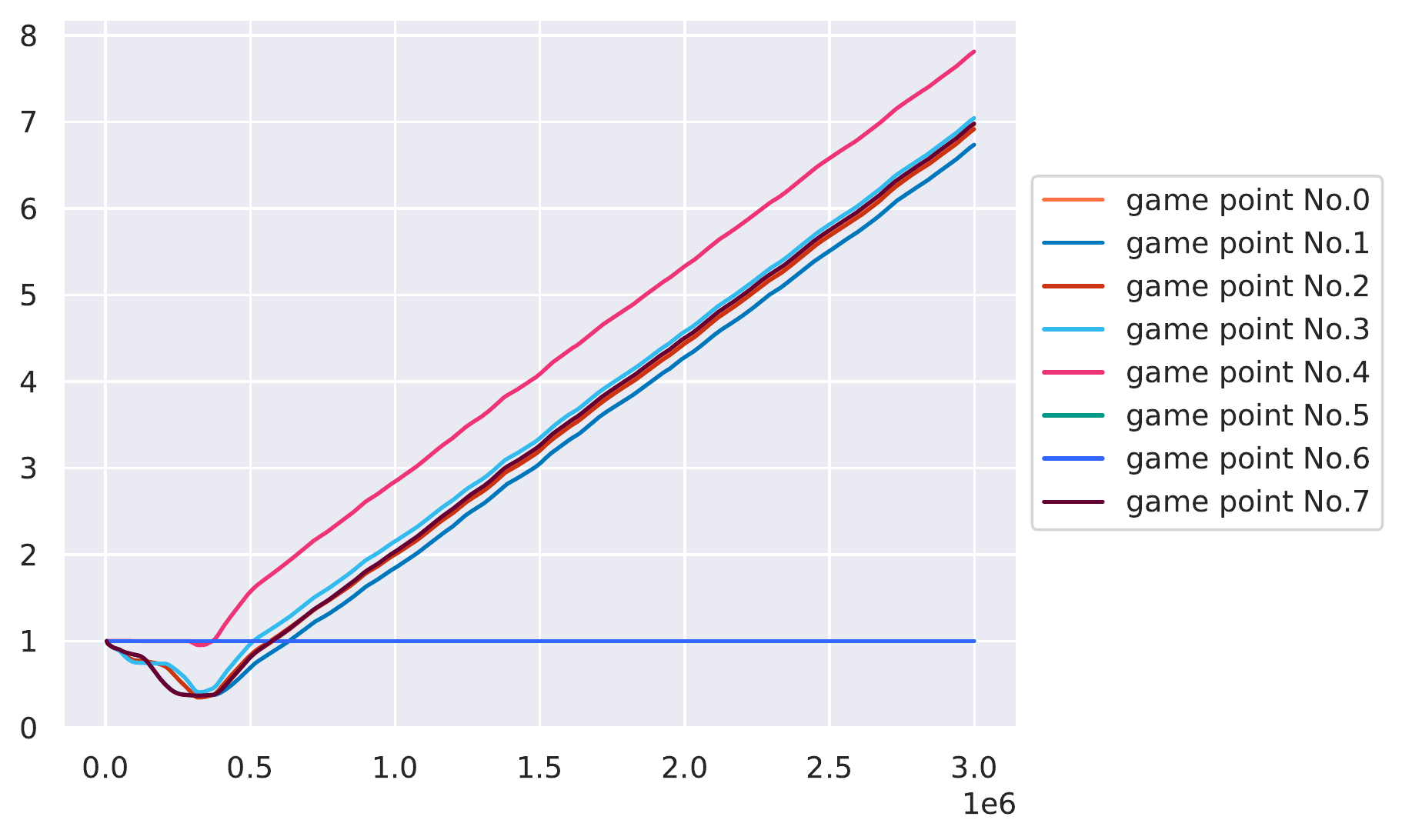}
    \label{swimmer_w}
    \end{minipage}
    }
    \centering
    \caption{The intrinsic rewards of each event generated by our MBRD method in the tested domains. (The x-axis is the number of training steps and the y-axis is the weights of each event for the intrinsic rewards during the training process).}
    \label{train_info_reward}
\end{figure*}

\begin{figure}[t]
\centering
\subfigure[Actor loss of Foraging]{
\begin{minipage}[t]{\linewidth}
\centering
\includegraphics[width=0.49\linewidth]{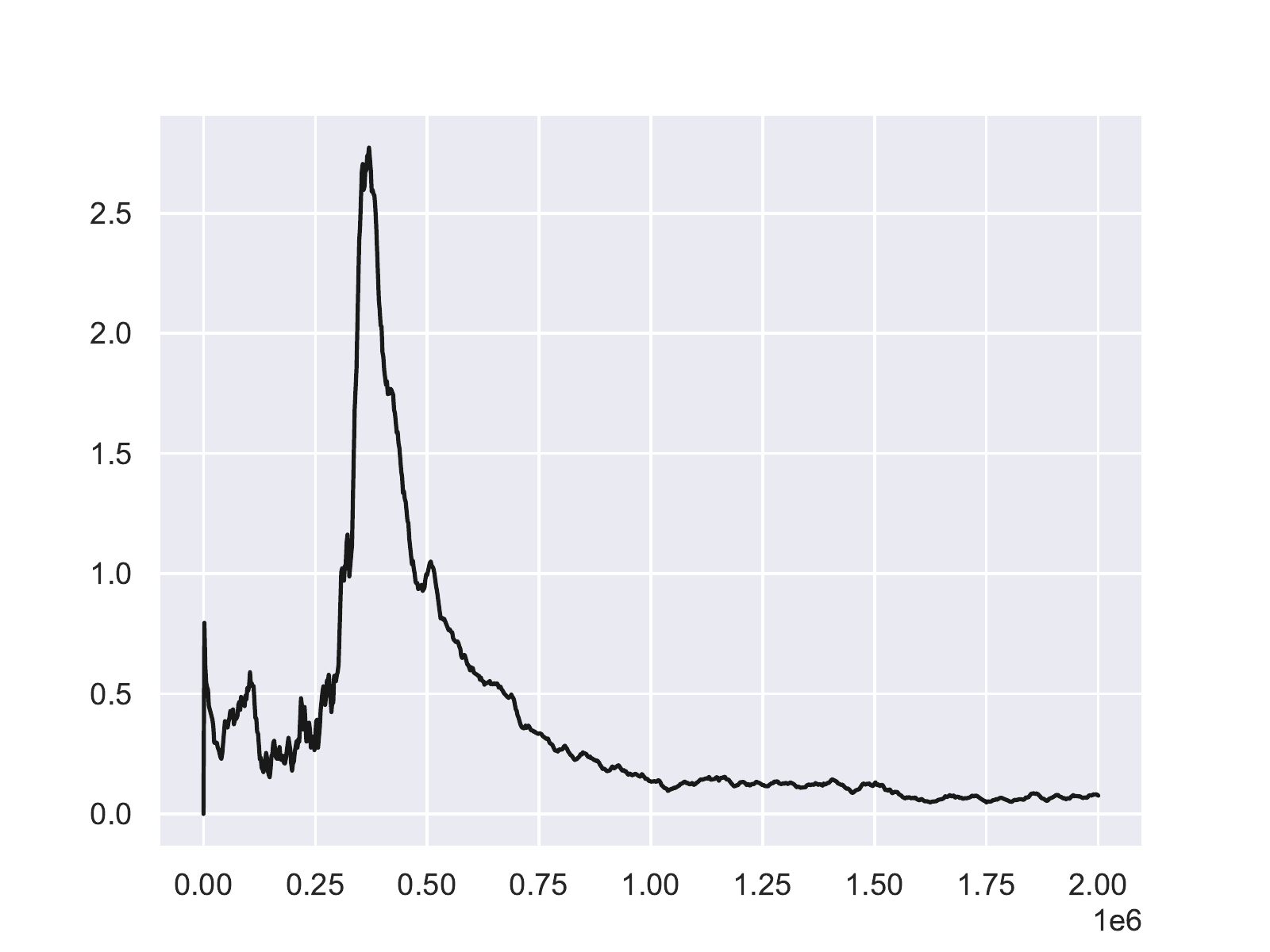}
\includegraphics[width=0.49\linewidth]{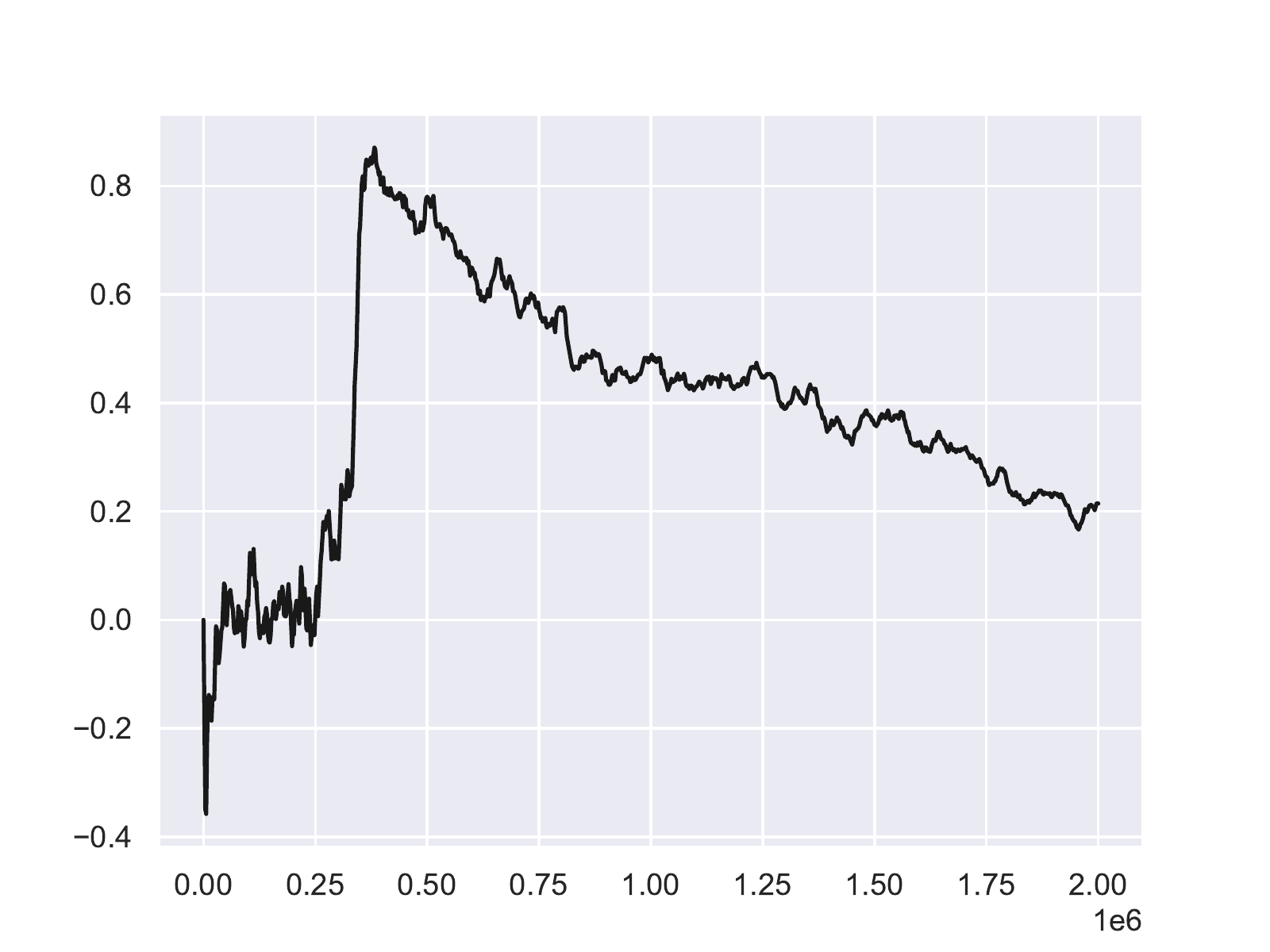}
\label{foraging_loss}
\end{minipage}
}
\subfigure[Actor loss of Hungry-Thirsty]{
\begin{minipage}[t]{\linewidth}
\centering
\includegraphics[width=0.49\linewidth]{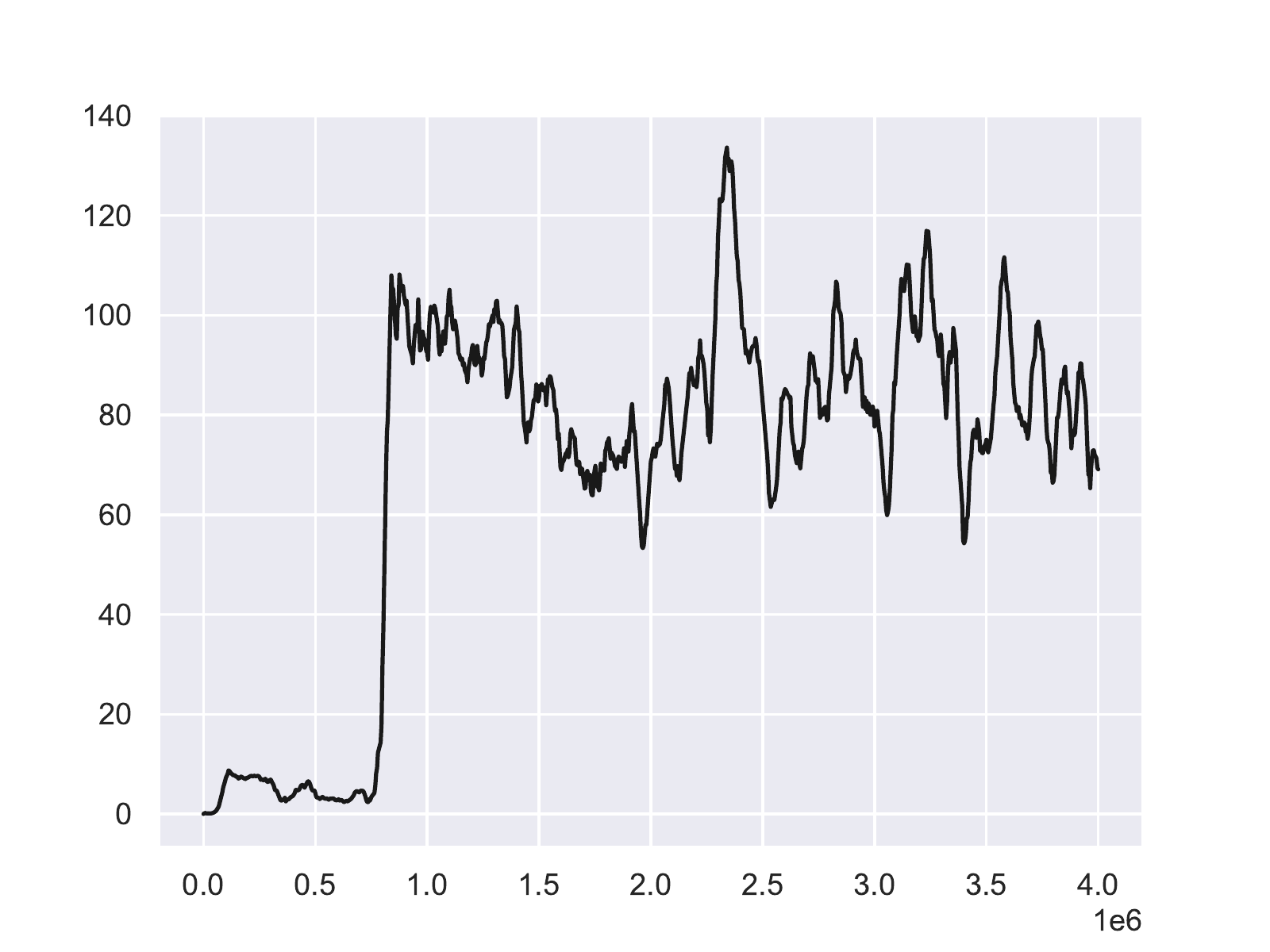}
\includegraphics[width=0.49\linewidth]{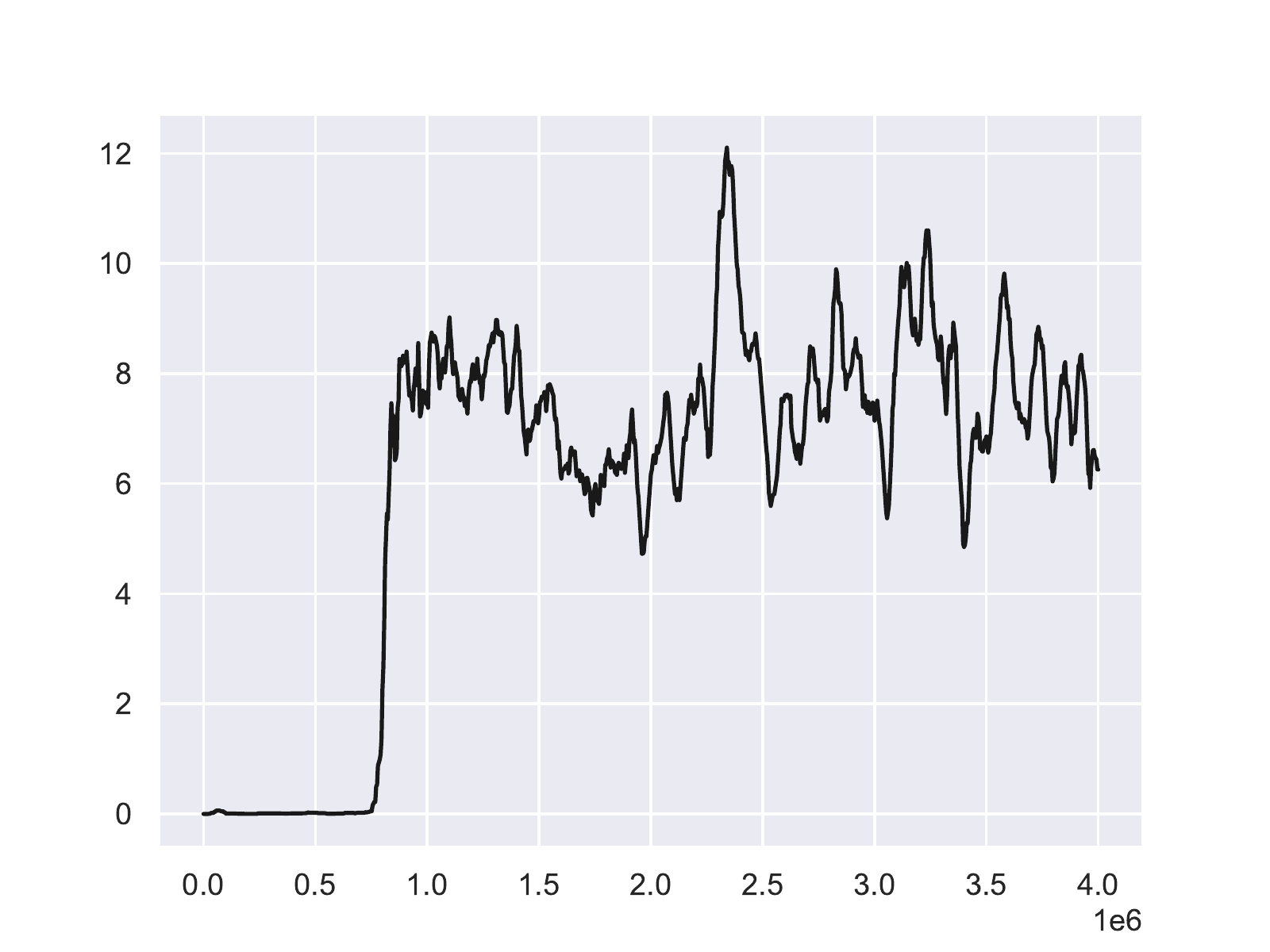}
\label{hungry_loss}
\end{minipage}
}
\subfigure[Actor loss of Fight Monster]{
\begin{minipage}[t]{\linewidth}
\centering
\includegraphics[width=0.49\linewidth]{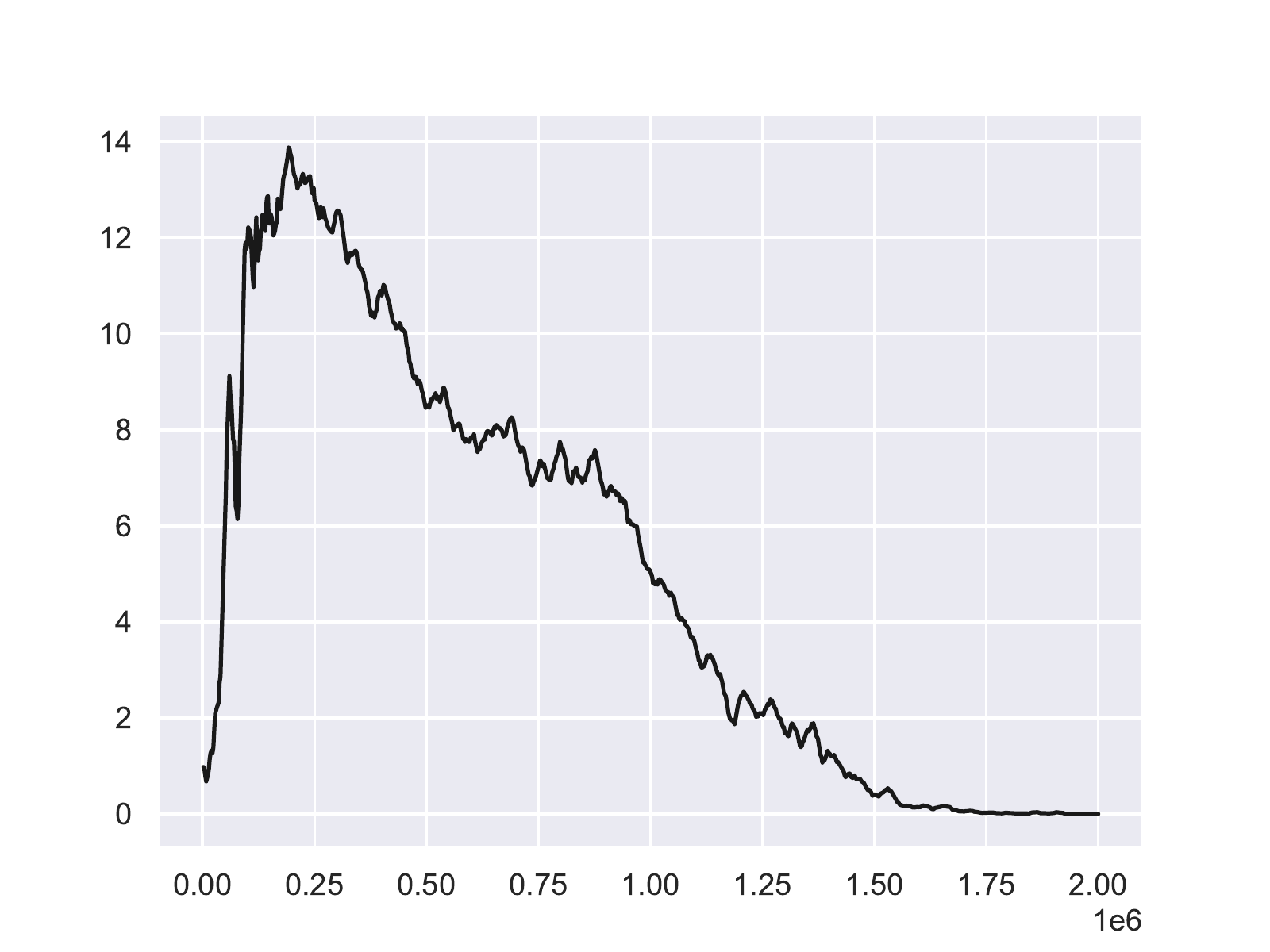}
\includegraphics[width=0.49\linewidth]{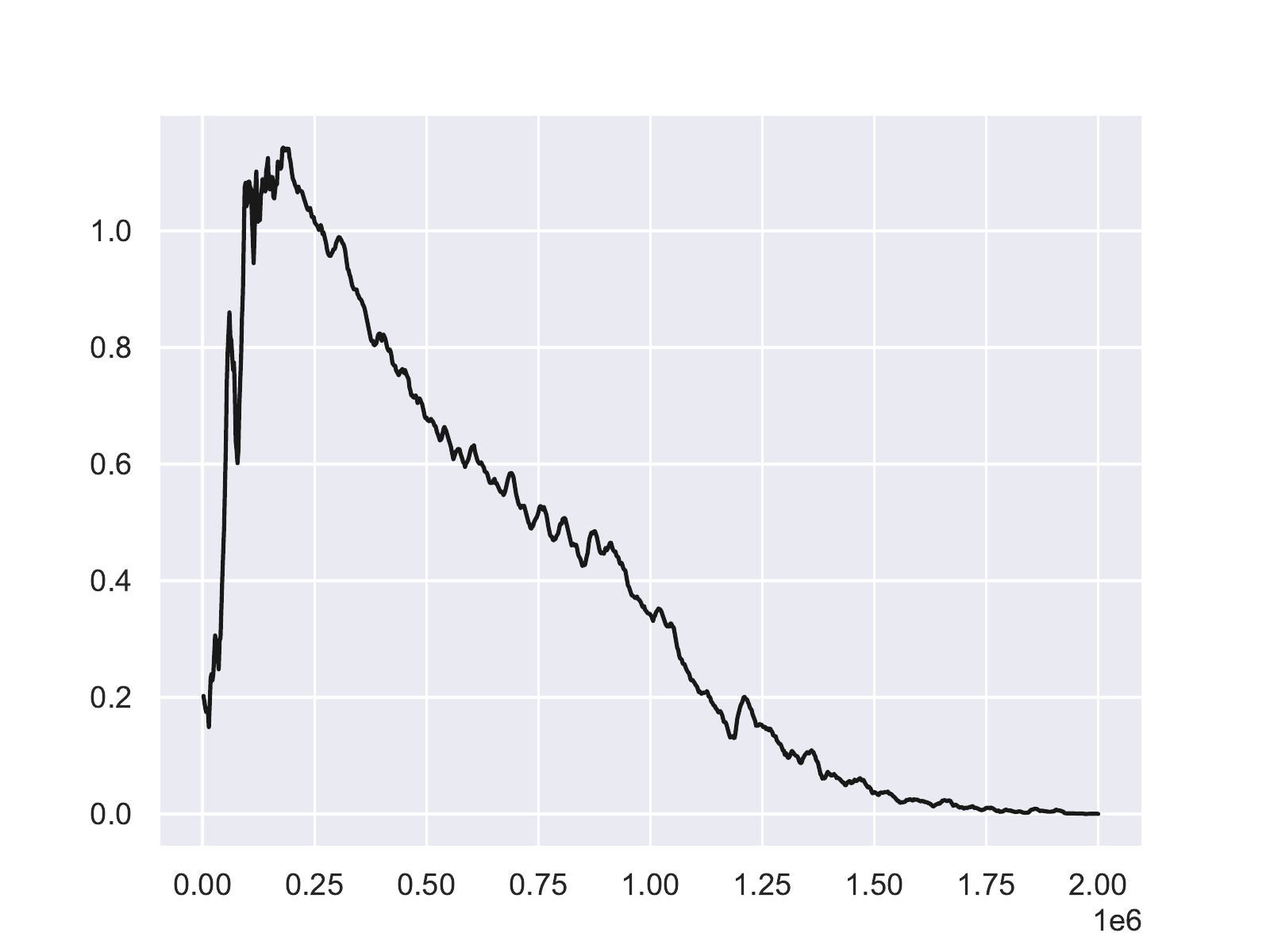}
\label{reason_loss}
\end{minipage}
}
\caption{Actor loss of virtual update process (The left figure is the loss of the inner optimization problem using the intrinsic rewards, while the right figure is the loss of the outer optimization problem using the extrinsic rewards. The x-axis is the number of training steps and the y-axis is the value).}
\label{train_info}
\end{figure}

\section{Experiments}

In this section, we conducted experiments to evaluate the advantages of designing intrinsic rewards for training RL agent instead of using extrinsic rewards directly. Specifically, we will show that an appropriate intrinsic reward function is beneficial for handling challenging RL problems of delayed reward, exploration, and credit assignment. In the experiments, we confirmed that our method performed better than the baselines on these problems in the tested domains.

We use PPO \citep{schulman2017proximal} as the underlying policy optimization algorithm for all baselines, where the main difference among these algorithms is in the way of reward generation. Notice that We do not make any assumptions for $J_{inner}$, so that any RL algorithm can be used in optimizing $J_{inner}$. In the experiments, all the algorithms are implemented in Python with OpenAI gym and Tensorflow running on an i9-10900k CPU.

\subsection{Grid-World Domains}

%We first illustrate the problems that MBRD hopes to solve in three simple domains, and try to explain why MBRD is powerful in these experiments.

We test our method in the three grid-world domains as shown in Figure \ref{domain}. The size of the map is 5$\times$5 for all the grid-world domains. The basic settings are as follow:

\begin{itemize}
  \item \textbf{Foraging}: The agent move in a grid-world with apple and poison randomly located. When the agent reaches the position of [apple, poison], it will trigger an event in $\rho$: [agent eats apple, agent eats poison], with delayed extrinsic rewards $r\in$[1, -1] received 10 steps later.
  \item \textbf{Hungry-Thirsty} \cite{singh2009rewards}: The agent move in a grid-world with food and water randomly located. It can eat in a food location and drink in a water location with an event in $\rho$: [eat, drink]. After eating, it becomes non-hungry for one time step and receives extrinsic reward 1. After drinking it becomes non-thirsty for a period. When it is thirsty, the eating action fails.
  \item \textbf{Fight Monster}: The agent move in a grid-world with weapon, poison, and monster randomly located. When it fights with the monster, it will get 1 extrinsic reward with the weapon (buff), -1 extrinsic reward after touching the poison (debuff), and -0.1 extrinsic reward otherwise. The events are $\rho$: [agent gets buff, agent gets debuff, agent wins the fight with monster, agent loses the fight with monster, agent draws the fight with monster].
\end{itemize}

We compared our method with the following state-of-the-art baselines for automatic reward design:
\begin{itemize}
\item \textbf{PPO} (Proximal Policy Optimization)~\citep{schulman2017proximal}: RL without any auxiliary rewards methods.
\item \textbf{CB} (Count-Based Bonus): A method gives bonuses ($\sqrt{1/n}$) based on the number of occurrences of events.
\item \textbf{PBRS} (Potential-Based Reward Shaping)~\citep{ng1999policy}: A reward shaping method with policy invariance ensured.
\item \textbf{LIRPG} (Learning Intrinsic Rewards for Policy Gradient)~\citep{zheng2018learning}: A gradient-based method for learning intrinsic rewards.
\end{itemize}

%All of this baselines above and our MBRD have been tested in three simple domains, but only PPO, LIRPG, and our MBRD have been tested in Hopper-v2 and Swimmer-v2 because CB and PBRS is not suitable for these two domains.
%
%We do not make any assumptions for $J_{\text{inner}}$ because all of these methods above are reward design methods (except PPO), so that any reinforcement learning algorithm can be used in $J_{\text{inner}}$, and we use Proximal Policy Optimization  ~\citep{schulman2017proximal} algorithm in our experiment. Note that the algorithm used in virtual optimization process is also unlimited, we use policy gradient algorithm for convenience, and the compute for $G^{in}$ and $G^{ex}$ is the cumulative reward of a trace start from random step to the end. In the execution phase, we greedily select the action with the largest probability value to execute.

In the grid-world experiments, the base learner (PPO) adopts a two-layer policy network with 8 units in each layer and a two-layer value function network with 32 units in each hidden layer. Each hidden layer is followed by a ReLU nonlinearity. The policy and value networks are updated in every 1024 steps and one such update contains 50 optimizing epochs with batch size 1024. The discount rate $\gamma$ is 0.999 in all tests. We train two million steps in both the Foraging and Fight Monster domains, and four million steps in the Hungry-Thirsty domain. The maximum episode length is 200 for all the three domains. The hyper-parameter $\beta$ is set differently in different domains. Specifically, we set $10^{-2}$ in the domain with high-frequency game points(i.e., Hungry-Thirsty) and $10^{-3}$ in the domain with low-frequency game points(i.e., Foraging and Fight Monster).

\subsubsection{Overall Performance.}

Figures \ref{foraging} , \ref{hungry} and \ref{reason} show that our method outperformed the baselines in all the three domains with the faster convergence speed and significantly higher cumulative extrinsic rewards. This is achieved by the advantages of our method in handling RL problems with delayed reward, exploration, and credit assignment. We will analyze each of them in the following sections.

%Generally, reinforcement learning using extrinsic reward signals (e.g. win or loss) to identify the quality of intrinsic reward functions, while our algorithm infers the motivation of reward signals, with the benefit is that the agent knows clearly which things need to be done to win higher extrinsic rewards.

\begin{figure}[t]
    \centering
    \includegraphics[width=1.0\linewidth]{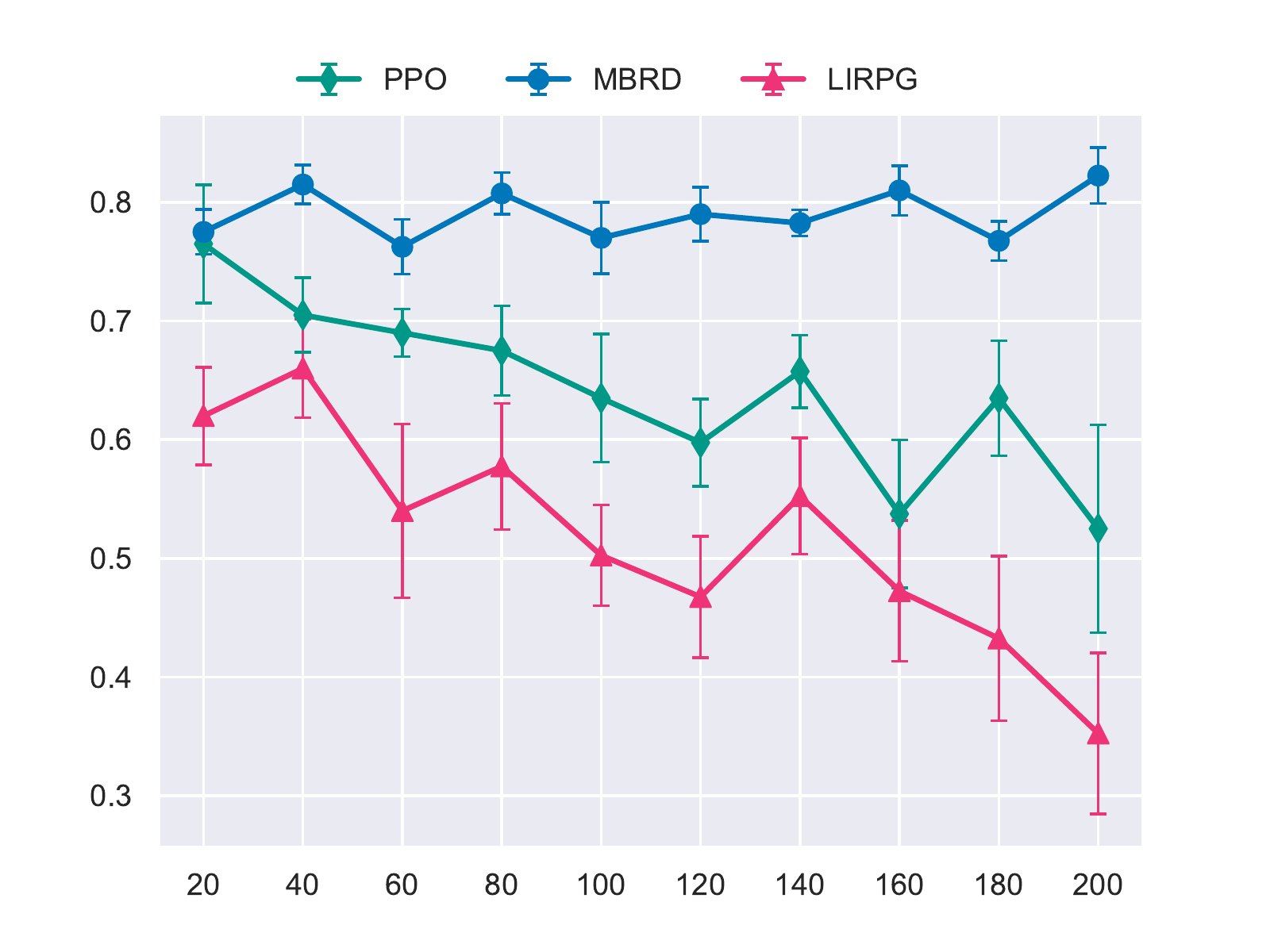}
	\caption{Cumulative extrinsic rewards (y-axis) with different maximum episode lengths (x-axis) in Foraging. }
    \label{step}
\end{figure}

\subsubsection{Delayed Rewards.}

In the Foraging domain, we can see from Figure \ref{foraging} that our method has substantial advantage over the baselines. Notice that the extrinsic rewards of this domain are delayed by 10 time steps. Existing methods such as LIRPG use the overall extrinsic cumulative rewards to optimize the intrinsic reward function. Hence they are more sensitive to the delays especially with longer episode lengths. This is confirmed by our results shown in Figure \ref{step}. The learned intrinsic rewards for the Foraging domain is shown in Figure \ref{foraging_w}. We can see that the agent learns correctly that eating apple is positive and eating poison is negative.

\begin{figure}[t]
    \centering
    \includegraphics[width=1.0\linewidth]{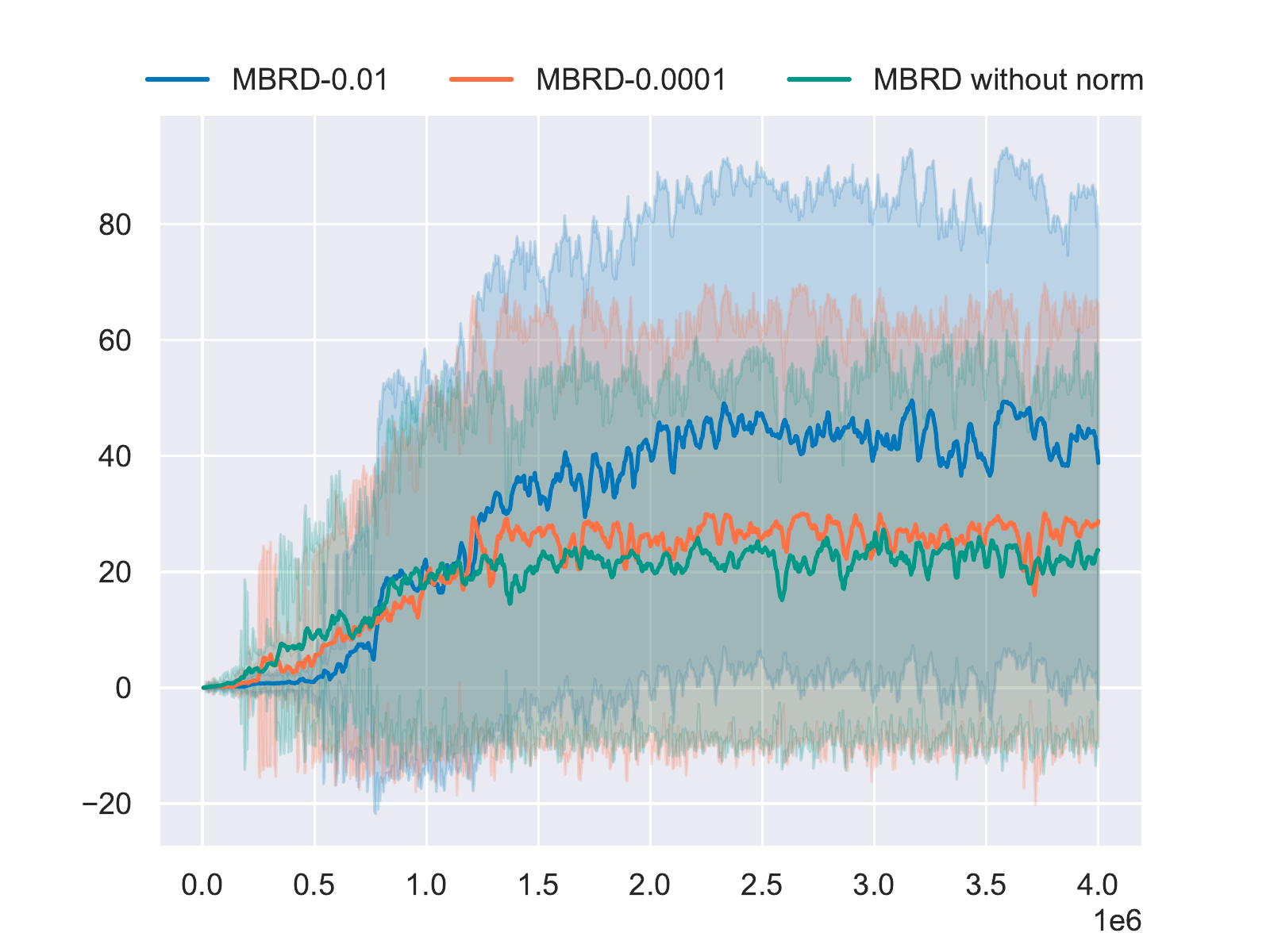}
	\caption{Ablation experiments in Hungry-Thirsty with different hyper-parameter of $\beta$ (i.e., $\beta=$ 0.01, 0.0001, and 0).}
    \label{hungry_ab}
\end{figure}

\subsubsection{Exploration.}

In the Hungry-Thirsty domain, the agent only receives extrinsic rewards with the ``eat'' action, which will fail if the agent is thirsty. The agent will become thirsty if it did not ``drink'' for a period of time. However, the ``drink'' action has no extrinsic rewards. Most of existing reward design methods depend on the actions with extrinsic rewards (e.g., ``eat'') and have difficulty on handling the actions without extrinsic rewards (e.g., ``drink''). In contrast, we use the term $\beta\nabla_{\phi}||\pmb{w}_{\phi} - \pmb{w}_{init}||^2$ in our optimization objective (i.e., $\beta \nabla_{\phi} ||z_{in}||$) to encourage exploration on the actions without extrinsic rewards. The benefit of doing this is shown in Figure \ref{hungry}. We also conduct ablation experiments with different $\beta$ to show its effectiveness in Figure \ref{hungry_ab}. The learned intrinsic rewards for the Hungry-Thirsty domain is shown in Figure \ref{hungry_w}. We can see that the agent correctly knows that drinking is equally important with eating.

\subsubsection{Credit Assignment.}

In the Fight Monster domain, the extrinsic rewards of fighting with the monster depend on whether the agent has taken the weapon (buff), touched the poison (debuff), or done nothing. In other words, the credit comes from the previous actions instead of the fighting action. Most of existing reward design methods fail to assigning the credit to the correct action. As we can see from Figure \ref{reason}, our method performs significantly better than the baselines. This is due to the fact that we compute the intrinsic rewards based on the motivations instead of directly the extrinsic rewards. The learned intrinsic rewards for the Fight Monster domain is shown in Figure \ref{reason_w}. We can see that the agent correctly learn that getting the buff is beneficial for fighting with the monster.

In Figures \ref{foraging_loss}, \ref{hungry_loss}, and \ref{reason_loss}, we visualize the losses of the inner and outer optimization problems as the training step increases. Notice that these losses correspond to the motivations $z_{in}$ and $z_{ex}$. As we can see from the figures, the shapes of the two losses tend to match each other. This indicates that the distance between the motivations $z_{in}$ and $z_{ex}$ becomes close to each other during the training process. This is what we expect in the two-tire optimization problems.

\subsection{Mujoco Domains}

In order to evaluate our MBRD method in more complex domains, we conducted experiments in two Mujoco domains. We slightly modified their settings to make them more challenging for automatic reward design as follow:
\begin{itemize}
  \item \textbf{Hopper-v2}: In this domain, we only modified the way that the extrinsic rewards are given. Specifically, the agent only receives the cumulative reward of the original problem at the end of a trajectory and 0 for the intermediate time steps. Hence the agent no longer gets the environment's evaluation of the quality of each step of its action. In other words, the rewards of this problem become sparse and delayed so the reward design is important for the success of this domain. We splits the overall range of the original rewards into 11 successive regions called game points $\rho$: [0 .. 10] as in Table \ref{env_hopper}. An event will be triggered if the original reward of the agent's action is in one of the regions. Although the agent only receives the cumulative reward at the end of a trajectory, it still can observe the event at each time step. Therefore, we can compute the intrinsic rewards based on such information.
  \item \textbf{Swimmer-v2}: We modified the rewards of this domain similar to the Hopper-v2 domain. Here, we split the original reward of the environment into two parts: the forward reward and the control reward. The game points are specified based on the range of the forward and control reward in the original problem. In more details, we set $\rho$: [0, 1, 2,  3, 4] for the forward reward and [5, 6, 7] for the control reward as in Table \ref{env_swimmer}, depending on the domain.
\end{itemize}

In the Mujoco domains, the base learner (PPO) adopts a two-layer policy network with 64 units in each layer and a two-layer value function network with 64 units in each hidden layer. The policy and value networks are updated in every 20000 steps and one such update contains 5 optimizing epochs with batch size 1024. We train three million steps and maximum episode length is 1000 in both the Hopper-v2 and Swimmer-v2 domains. The discount rate $\gamma$ is 0.99 in all tests. The hyper-parameter $\beta$ is $10^{-3}$.

\begin{table*}[t]
\center
\caption{Game Points of Hopper-v2}
\label{env_hopper}
\begin{tabular}{ccccccc}
\toprule
Game Point No. & 0 & 1 & 2 & 3 & 4 & 5 \\
\midrule
Reward Range & (-$\infty$, -3) & [-3, -2) & [-2, -1) & [-1, -0.5) & [-0.5, 0) & [0, 0.5)  \\
\midrule
Game Point No.  & 6 & 7 & 8 & 9 & 10 \\
\midrule
Reward Range & [0.5, 1) & [1, 1.5) & [1.5, 2) & [2, 3) & [3, $\infty$) \\
\bottomrule
\end{tabular}
\end{table*}

\begin{table*}[t]
\center
\caption{Game Points of Swimmer-v2}
\label{env_swimmer}
\begin{tabular}{cccccc}
\toprule
Game Point No. & 0 & 1 & 2 & 3 & 4 \\
\midrule
Forward Reward Range & $(-\infty, 0)$ & $[0, 0.5)$ & $[0.5, 1)$ & $[1, 2)$ & $[2, \infty)$  \\
\midrule
Game Point No. & 5 & 6 & 7 \\
\midrule
Control Reward Range & $[-\infty, -2)$ & $[-2, -1)$ & $[-1, 0)$ \\
\bottomrule
\end{tabular}
\end{table*}

The overall performance of our method is shown in Figures \ref{hopper} and \ref{swimmer}. As we can see from the figures, our method outperformed all the baselines in both domains.

In Figure \ref{hopper_w}, we show the intrinsic rewards for the Hopper-v2 domain learned by our method. Here, the agent learns that all game points except No. 0 and 1 are beneficial to improve extrinsic rewards in the early stage. Among them, game points No. 6 and 7 are most useful so the intrinsic rewards corresponding to them increase most quickly. In the final stage, the agent leans that game points No. 5-8 that are beneficial. Note that the original rewards corresponding to these game points are from 0.5 to 3. This is consistent with the original reward model.

In Figure \ref{swimmer_w}, we show the intrinsic rewards for the Swimmer-v2 domain learned by our method. In this domain, the agent quickly learns that game point No. 4 corresponding to the forward reward $[2, \infty)$ and game point No. 7 corresponding to the control reward $[-1, 0)$. They are indeed the key to improving the agent's performance, which also fits the original reward model.

\section{Related Work}

In this section, we briefly review the previous work in the literature that are most related to our method.

\subsection{Reward Shaping}

Potential-based reward shaping (PBRS) \cite{ng1999policy} ensured that the optimal policy of the original problem is not violated by the shaping reward function. Other attempts have been made to extend the potential-based shaping method to state-action shaping functions \cite{wiewiora2003principled} and dynamic shaping \cite{devlin2012dynamic}. \citet{harutyunyan2015expressing} proposed the method of transforming arbitrary reward functions to potential-based advice, which is used in inverse RL \cite{suay2016learning}. \citet{zou2019reward} proposed a meta-learning algorithm based on the idea of PBRS providing an initial approximation to the value function.

\subsection{Intrinsically Motivated RL}

%~\citet{schmidhuber1991curious} proposed the idea of rewarding agents by how quickly an agent¡¯s environment model is improving.
\citet{chentanez2005intrinsically} introduced intrinsically motivated RL using predict error as reward signals.
%And the empowerment function can function as an intrinsic reward signal which is an information-theoretic measure of an agent¡¯s ability to control its environment ~\citep{klyubin2005empowerment}.
Recently, intrinsic rewards are often used to encourage exploration, such as: ``curiosity'' \cite{mohamed2015variational,houthooft2016vime} that uses prediction error as reward signal, and ``visitation counts''  \cite{bellemare2016unifying,fu2017ex2} that discourages the agent from revisiting the same states. These intrinsic rewards bridge the gaps among sparse extrinsic rewards, by guiding the agent to efficiently explore the environment to find the next extrinsic reward \cite{burda2018large}.

\subsection{Optimal Reward Design}

\citet{singh2010intrinsically} proposed the optimal reward framework which has shown that intuition alone is not always adequate to devise a good reward signal. \citet{sorg2010reward} proposed the policy gradient for reward design  algorithm --- a scalable algorithm that only works with lookahead-search based planning agents. \citet{zheng2018learning} introduced a gradient-based method for learning intrinsic rewards for model-free RL agents. \citet{jaderberg2019human} proposed a method to build intrinsic rewards based on game points and use evolution strategies to find the optimal intrinsic reward parameters. \citet{venkattaramanujam2019self} generates internal signals by learning a distance metric to reward agents for getting closer to their target image-based goals. \citet{colas2020language} learn language-conditioned reward functions from language descriptions of autonomous exploratory trajectories respectively.

%Optimal rewards aim to automate the trial-and-error search for a good signal \cite{sutton2018reinforcement}.
%Due to the evolutionary strategy results in high computational cost, some researchers try to optimize the intrinsic reward function by gradient ascent.
%~\citet{guo2016deep} extended PGRD to a deep reinforcement learning version, which enabled image input and used in Atari games.

Unlike previous work that treats intrinsic rewards as a part of the policy, we extract the underlying goal of maximizing certain rewards, and shape intrinsic rewards by minimizing the distance between extrinsic and intrinsic motivations.

%\subsection{Goal-Conditioned Intrinsically Motivated Goal Exploration Processes (GC-IMGEP)}
%
%GC-IMGEP is the family of algorithms that train agents to generate and pursue their own goals by training goal-conditioned policies, a few approaches try to learn their own goal-conditioned reward function \cite{colas2020intrinsically}. \citet{colas2020language} learn language-conditioned reward functions from language descriptions of autonomous exploratory trajectories respectively.
%\citet{venkattaramanujam2019self} generates internal signals by learning a distance metric to reward agents for getting closer to their target image-based goals.
%Unlike GC-IMGEP, our purpose is not to achieve some goals, but to build intrinsic rewards through extrinsic rewards under the optimal reward framework.

%explain How extrinsic rewards affect the construction of intrinsic rewards?
%explain the result in Fight monsters(performance and w, rho, gradient curve)

\section{Conclusions}

We proposed a novel method named MBRD for shaping goal-consistent intrinsic rewards. We introduced the concept of motivation which captures the underlying goal of maximizing certain rewards, and shapes the intrinsic rewards by minimizing the distance between intrinsic and extrinsic motivations. The experimental results show that the MBRD method has advantages in handling challenging RL problems such as delayed reward, exploration, and credit assignment. In addition to being directly used to shape a new intrinsic reward function, it can also be used to optimize existing reward functions, such as using human comments as extrinsic rewards to fine-tune existing reward signals. The flexible framework of MBRD encourages more potential future works, designing new delicate methods to extract the goal of rewards may contribute to wider applications, and automatic detection for game points is also a potential research direction.

\bibliography{citations}

\end{document}